\documentclass[11pt]{article}

\usepackage[a4paper,margin=1in]{geometry}
\usepackage[T1]{fontenc}
\usepackage[utf8]{inputenc}
\usepackage{lmodern}
\usepackage{microtype}
\usepackage{amsmath,amssymb,amsfonts,amsthm,mathtools,mathrsfs,bm}
\usepackage{enumitem}
\usepackage[hidelinks]{hyperref}
\usepackage{graphicx}
\usepackage{booktabs}
\usepackage{float}
\allowdisplaybreaks
\newtheorem{theorem}{Theorem}[section]
\newtheorem{proposition}[theorem]{Proposition}
\newtheorem{corollary}[theorem]{Corollary}
\newtheorem{lemma}[theorem]{Lemma}

\newtheorem{problem}[theorem]{Problem}
\theoremstyle{definition}
\newtheorem{definition}[theorem]{Definition}
\newtheorem{remark}[theorem]{Remark}

\newcommand{\N}{\mathbb{N}}
\newcommand{\R}{\mathbb{R}}
\newcommand{\F}{\mathbb{F}}

\newcommand{\Prb}{\mathbb{P}}
\newcommand{\E}{\mathbb{E}}
\newcommand{\Mod}{\operatorname{Mod}}
\newcommand{\Th}{\operatorname{Th}}
\newcommand{\Cn}{\operatorname{Cn}}
\newcommand{\Min}{\operatorname{Min}}
\newcommand{\Row}{\operatorname{Row}}
\newcommand{\Ker}{\operatorname{Ker}}
\newcommand{\rank}{\operatorname{rank}}
\newcommand{\tp}{\operatorname{tp}}

\newcommand{\Sent}{\operatorname{Sent}}

\newcommand{\argmin}{\operatorname*{arg\,min}}
\newcommand{\Lang}{\mathcal{L}}
\newcommand{\K}{\mathcal{K}}
\newcommand{\A}{\mathcal{A}}
\newcommand{\M}{\mathfrak{M}}
\newcommand{\Nstr}{\mathfrak{N}}
\newcommand{\up}{\uparrow}
\newcommand{\restr}{\upharpoonright}
\newcommand{\TV}{\operatorname{TV}}

\newcommand{\TD}{\operatorname{TD}}
\newcommand{\Sep}{\operatorname{Sep}}

\title{Finite Certificates for In-Context Determinacy\\
and a Threshold Theory of Emergence in Language Models}
\author{
Faruk Alpay\thanks{Corresponding author: \texttt{alpay@lightcap.ai}}\\
Department of Computer Engineering\\
Bahcesehir University, Istanbul, Turkey\\
\texttt{faruk.alpay@bahcesehir.edu.tr}
\and
Hamdi Alakkad\\
Department of Artificial Intelligence Engineering\\
Bahcesehir University, Istanbul, Turkey\\
\texttt{hamdi.alakkad@bahcesehir.edu.tr}
}
\date{}

\begin{document}
\maketitle

\begin{abstract}
We study two verification problems for context-conditioned language-model behavior by replacing benchmark labels with finite semantic certificates. The first problem is finite determinacy: when do examples in a context force the answer to a query without changing model parameters? For finite-field linear task families we prove an exact row-space criterion, compute the residual hypothesis count, give the full-identification curve $I_d(n)=\prod_{i=0}^{d-1}(1-Q^{i-n})$, and derive a query-local determination curve for a fixed nonzero query. We also show that extracting a smallest forcing subcontext is NP-complete, even for binary outputs. The second problem is threshold emergence: when does an apparent benchmark jump indicate a semantic transition rather than a discontinuity of the scoring map? We prove an anti-mirage theorem separating thresholded metrics from semantic confidence, and a rate-sensitive crossing bound $\lambda_\tau=(a/(1-\tau))^{1/\alpha}$ for latent commitments becoming visible above threshold. The common semantic object is the confidence functional $s_{\lambda,c}(\varphi)=\mu_{\lambda,c}([\![\varphi]\!])$ on definable events. We show that it is a Boolean probability measure, equivalently a Keisler measure on the relevant type space, whose measure-one formulas form a proper filter and whose Stone-space representation is invariant under definitional expansion. The resulting calculus gives reusable objects: finite context certificates, pair-separator hitting sets, query teaching dimension, prompt-preservation criteria, and scale-limit witnesses. An ancillary artifact reproduces the finite-field and threshold calculations by exact-arithmetic scripts and records the emitted data files used by the figures. To probe whether the certificates these theorems isolate are within reach of trained systems, we also run a deterministic certificate-emission benchmark over a weak-to-mid panel of contemporary language models, scored by exact match and by a graded proxy. The exact score reproduces the predicted threshold jump---it stays near zero on a multi-field certificate until a single system crosses, while the graded confidence rises smoothly---an instance of the anti-mirage theorem on trained systems, sharpened by a conjunction count that inflates the crossing scale by $k^{1/\alpha}$. Because each family criterion is a sound, oracle-free checker, the same machinery defines an aversive closed loop whose accepted set is sound and non-decreasing; its reach is governed by whether a generator can act on the checker's directional feedback, so it provably erases a metric-artifact threshold while leaving a genuine arithmetic gap intact, with no reference answer entering the loop.
\end{abstract}

\section{Introduction}

This paper treats context-conditioned language-model behavior as a finite certification problem. The first target is \emph{determinacy}: given examples in the context and a query, when is the answer forced by the semantic constraints rather than merely favored by the decoder? In finite-field linear task families the answer is exact. A query is forced precisely when it lies in the row space of the example design matrix; the number of remaining latent parameters is $Q^{d-r}$ when the observed matrix has rank $r$; full identification after $n$ random examples has probability
\[
I_d(n)=\prod_{i=0}^{d-1}\bigl(1-Q^{i-n}\bigr),
\]
and a fixed nonzero query has its own determination curve obtained by averaging over the rank distribution. Extracting the smallest forcing subcontext is NP-complete. The second target is \emph{threshold emergence}: a benchmark can jump because the metric thresholds a smooth semantic confidence curve. We prove an anti-mirage theorem for this separation (Theorem~\ref{thm:mirage}) and a rate-sensitive crossing-scale bound $\lambda_\tau=(a/(1-\tau))^{1/\alpha}$ for latent commitments becoming visible above threshold (Theorem~\ref{thm:ratebound}). These statements do not assert that every trained model is exactly linear or that every benchmark jump is artificial. They provide a certification language in which those claims can be tested. Section~\ref{sec:verify-numeric} reproduces the quantitative theorems by exact-arithmetic scripts, and Section~\ref{sec:predictions} states controlled falsification protocols for trained models.

Large language models are evaluated through an increasingly stable empirical pattern: a fixed trained system receives a finite context, the context changes the distribution over outputs, and larger systems make some context-conditioned behaviors more reliable. GPT-3 made this pattern operational by reporting zero-shot, one-shot, and few-shot performance across translation, question answering, cloze completion, arithmetic, and SuperGLUE-style evaluations, all without parameter updates at test time \cite{Brown2020}. Scaling-law work then showed that cross-entropy loss can vary smoothly with scale over broad regimes, which separates the growth of capability from any single benchmark threshold \cite{Kaplan2020}. Compute-optimal scaling refined the picture by showing that data allocation and model size jointly determine the observed loss curve \cite{Hoffmann2022}. PaLM continued the same trend across large-scale language, reasoning, and multilingual benchmarks \cite{Chowdhery2022}.

The later GPT line made the semantic ambiguity harder to ignore rather than easier. GPT-4 reported human-level performance on several professional and academic benchmarks, including a simulated bar exam result around the top decile of test takers \cite{OpenAI2023GPT4}. GPT-4.1 reported large improvements on coding and instruction-following evaluations, including 54.6\% on SWE-bench Verified and 38.3\% on MultiChallenge \cite{OpenAI2025GPT41}. GPT-5 was presented as a system combining a fast model, a deeper reasoning model, and a router, with reported developer-facing scores such as 74.9\% on SWE-bench Verified and 88\% on Aider Polyglot \cite{OpenAI2025GPT5Dev}. GPT-5.5 further emphasized complex professional work, reporting 84.9\% on GDPval, 78.7\% on OSWorld-Verified, and 98.0\% on Tau2-bench Telecom without prompt tuning \cite{OpenAI2026GPT55}. These numbers do not by themselves say what kind of semantic object a prompt denotes, what it means for examples in the context to determine an answer, or when a visible benchmark jump corresponds to a genuine semantic transition.

This is the problem addressed here. Empirical evaluation supplies distributions, accuracies, and thresholded scores. Verification needs a different object: a logic in which a behavior can be stated as a formula, a context can be interpreted as an update, and a scale trend can be tested against a limiting semantic commitment. Without that separation, three distinct claims are easily conflated. First, a model may assign high probability to a correct output without semantically entailing it. Second, a prompt may override a previous default without adding an ordinary monotone axiom. Third, a benchmark score may jump because the observation map is discontinuous even though the underlying confidence changes gradually.

Existing theories of in-context learning explain important parts of this picture but do not define a common semantic target. One line treats in-context behavior as latent Bayesian inference over tasks \cite{Xie2022}. Another asks whether transformer computations implement standard learning algorithms during the forward pass \cite{Akyurek2023}. Gradient-descent interpretations give an operational account of a related phenomenon \cite{VonOswald2022}. Work on simple function classes identifies sample sequences that can be learned in context \cite{Garg2022}. Mechanistic studies explain circuit-level regularities such as induction heads \cite{Olsson2022}. These accounts are valuable, but they do not by themselves produce a first-order object whose consequences can be studied by compactness, diagrams, filters, or ultraproducts.

The corresponding semantic ingredients already exist in logic and formal pragmatics. Montague semantics treats meaning with model-theoretic precision \cite{Montague1973}. Stalnaker's context sets make assertion a restriction of possibilities \cite{Stalnaker1978}. Lewisian scorekeeping makes discourse state dynamic \cite{Lewis1979}. Heim's file-change semantics gives a formal discipline for context growth \cite{Heim1982}. Gricean pragmatics explains why literal content and intended force diverge \cite{Grice1989}. Rational speech-act models add a probabilistic account of pragmatic inference \cite{FrankGoodman2012}. Nonmonotonic logics and preferential models show how adding information can change what is selected as normal or preferred \cite{Kraus1990}. Belief revision supplies another mathematical language for update under priority and inconsistency management \cite{AGM1985}.

The contribution of this paper is to turn the LLM-specific problem into a model-theoretic one. For each scale $\lambda$ we use a tuple
\[
\mathbb{G}_{\lambda}=(\Lang,\mathcal{K}_{\lambda},\mu_{\lambda},\mathsf{C}_{\lambda},\mathsf{U}_{\lambda},\mathsf{D}_{\lambda}),
\]
where $\mathcal{K}_{\lambda}$ is a class of structures, $\mu_{\lambda}$ is a probability measure on definable events, $\mathsf{U}_{\lambda}$ is a context update rule, and $\mathsf{D}_{\lambda}$ is a decoding kernel. This is not a claim that transformer activations literally contain first-order models. It is a semantic representation of behavioral commitments, just as automata, Kripke structures, and probabilistic transition systems represent behavior without copying physical implementation.

\paragraph{Contributions.}
\begin{enumerate}[leftmargin=2.2em]
\item \textbf{Identification curves for in-context determinacy.} In a finite-field linear task family, a query is forced exactly when it lies in the row space of the example design matrix; the residual hypothesis count is $|\F|^{\,d-\rank(A_n)}$; full identification has curve $I_d(n)=\prod_{i=0}^{d-1}(1-Q^{i-n})$; and fixed-query determination is obtained by averaging $(Q^r-1)/(Q^d-1)$ over the rank distribution (Theorems~\ref{thm:linear},~\ref{thm:linearcount},~\ref{thm:randomlinear},~\ref{thm:querylocal}).
\item \textbf{Certificate complexity.} Finite-context determination reduces to a pair-separator hitting problem (Theorem~\ref{thm:pairseparator}), and extracting a minimal forcing subcontext is NP-complete already for binary outputs (Theorem~\ref{thm:certnp}); this connects context length to teaching dimension \cite{GoldmanKearns1995} and mistake-bound dimension \cite{Littlestone1988}.
\item \textbf{A threshold theory of emergence.} Thresholded benchmark jumps do not imply semantic discontinuity (Theorem~\ref{thm:mirage}); latent commitments become $\tau$-manifest with a crossing-scale bound (Theorem~\ref{thm:ratebound}); a continuous metric preserves graduality (Proposition~\ref{prop:continuousobs}); and a $k$-field exact-match metric sharpens the apparent jump and inflates the crossing scale by $k^{1/\alpha}$ (Proposition~\ref{prop:conj}), an effect we then observe on trained systems (Section~\ref{sec:benchmark}). This reconciles the emergence \cite{Wei2022} and mirage \cite{Schaeffer2023} accounts as two regimes of one inequality.
\item \textbf{Semantic confidence as a Keisler measure.} The confidence functional $s_{\lambda,c}$ is a finitely additive probability on definable events, that is, a Keisler measure on the type space \cite{Keisler1987,Simon2015}; its measure-one formulas form a proper filter (Theorem~\ref{thm:filter}), it admits a Stone-space representation (Theorem~\ref{thm:stoneconfidence}), and it is invariant under definitional expansion (Theorem~\ref{thm:definvinv}).
\item \textbf{Prompts as preferential updates.} Prompt consequence is nonmonotonic under extension (Theorem~\ref{thm:nonmon}); fixed-preference fragments satisfy cautious monotony (Theorem~\ref{thm:cautious}); prompt extension admits an exact preservation criterion (Theorem~\ref{thm:preservation}), which explains order- and extension-sensitivity of prompting \cite{Lu2022}.
\item \textbf{Verification, falsification, and a certificate benchmark.} Every quantitative theorem is confirmed by controlled exact-arithmetic simulation (Section~\ref{sec:verify-numeric}); the three certificate families are then posed directly to a fixed panel of contemporary language models and graded against exact ground truth (Section~\ref{sec:benchmark}). Because each family criterion is a sound, oracle-free checker, verification doubles as a closed loop: a provably-sound, non-decreasing aversive refinement schedule (Proposition~\ref{prop:refine}) whose reach is governed by a generator's ability to act on the checker's directional feedback, so it erases a metric-artifact threshold but not a genuine arithmetic gap (Proposition~\ref{prop:reach}). Each prediction is also given a checkable protocol for trained models (Section~\ref{sec:predictions}).
\end{enumerate}

These objects are built to be reused, not only invoked. Semantic confidence, the identification curve $I_d(n)$, the pair-separator certificate, and the crossing-scale bound $\lambda_\tau$ are each stated so that they can be measured, bounded, and tested against trained models, rather than serving as a purely conceptual distinction.

The technical target is therefore not a taxonomy of LLM behavior. It is a collection of decision and invariance problems. When does a finite context determine a unique answer? How many examples are necessary before a query becomes determined? Which prompt extensions preserve all old consequences? Which conclusions survive a harmless change of semantic presentation? When does a benchmark threshold reveal a limit-theoretic fact, and when is it only an artifact of observation? The body of the paper answers these questions in the general compactness setting where possible, and in finite linear task families where exact numerical criteria can be proved.

\section{Semantic presentations of pre-trained LLMs}\label{sec:presentation}

We work in many-sorted first-order logic. Standard references for model theory include Hodges and Marker \cite{Hodges1993,Marker2002}. The signature $\Lang$ contains at least a sort $\mathbf{C}$ for context codes and a sort $\mathbf{Y}$ for outputs. Additional sorts may represent latent tasks, discourse referents, world states, proof objects, tool states, or internal semantic hypotheses.

The semantic layer is not intended to be a literal description of transformer activations. It is an external representation of behavior. In the same way that a probabilistic automaton may summarize a physical device without reproducing its circuitry, a semantic presentation summarizes the input-output commitments of a trained model by a probability-bearing class of structures.

\subsection{The separation problem}

The central verification failure mode is a category error between three levels of description. Let $q$ be a task instance, let $A(q,y)$ say that $y$ is an admissible answer to $q$, and let $R(y)$ be the benchmark scoring predicate. A model can satisfy
\[
P_{\lambda,c}(R)>\tau
\]
while failing to satisfy any logical uniqueness statement about $A(q,y)$. Conversely, the semantic state may entail a unique admissible answer while the decoder assigns nonzero probability to malformed strings, refusals, or tool calls. A third possibility is that $s_{\lambda,c}(\psi)$ changes gradually while the reported score $\Omega(s_{\lambda,c}(\psi))$ jumps because $\Omega$ is discontinuous.

This paper solves the separation problem by assigning a different mathematical object to each level:
\[
\begin{aligned}
&\text{context update} && \mathsf{U}_{\lambda},
&\qquad \text{semantic truth} && \M\models\varphi,\\
&\text{confidence} && s_{\lambda,c}(\varphi),
&\qquad \text{generation} && \mathsf{D}_{\lambda}.
\end{aligned}
\]
The four objects interact, but none reduces to the others. The subsequent definitions are designed to make that non-reduction provable.

\begin{definition}[Semantic answerability]
Let $A(x,y)$ be a formula with $x$ of query sort and $y$ of output sort. A query term $q$ is \emph{semantically answerable} at $(\lambda,c)$ if
\[
\exists y\,A(q,y)\in\Th_{\lambda}(c).
\]
It is \emph{uniquely answerable} if
\[
\exists y\,A(q,y)\in\Th_{\lambda}(c)
\qquad\text{and}\qquad
\forall y\forall z\bigl(A(q,y)\wedge A(q,z)\rightarrow y=z\bigr)
\in\Th_{\lambda}(c).
\]
\end{definition}

\begin{proposition}[Answerability and observed success are independent]\label{prop:independence}
Semantic answerability, unique semantic answerability, and high observed benchmark probability are pairwise non-equivalent in general.
\end{proposition}

\begin{proof}
For answerability without unique answerability, take a model class in which two distinct outputs satisfy $A(q,y)$ on a measure-one set. Then $\exists y\,A(q,y)$ is almost sure, but uniqueness is false almost surely. For unique answerability without high observed success, take a model class where a unique $A$-answer is entailed and choose a decoding kernel assigning most probability mass to outputs outside the benchmark scoring predicate $R$. For high observed success without semantic answerability, take a model class where $A(q,y)$ varies across structures with no measure-one witness, but choose a decoder that assigns high probability to some output satisfying $R$. These constructions satisfy the definitions and show that no implication holds without additional faithfulness assumptions.
\end{proof}

\subsection{Theoretical targets}

The paper is organized around four mathematical problems. They are stated here explicitly because they are the points at which empirical language-model evaluation usually stops being logically well-formed.

\begin{problem}[Finite determinacy]
Given a background theory $T_0$, an example stream $E$, a query $q$, and an output property $\chi(y)$, characterize when
\[
T_\omega\models \chi(f(q))
\]
can be witnessed by some finite prefix $T_N$.
\end{problem}

The compactness results in Section~\ref{sec:icl} solve the qualitative version of this problem for first-order properties. The linear task results in Section~\ref{sec:latenttasks} give exact rank and row-space certificates for a concrete family.

\begin{problem}[Certificate complexity]
For a class of tasks and a fragment $\mathcal F$, bound the least $N$ for which $T_N$ entails the target formula, or prove that no shorter prefix can do so.
\end{problem}

This problem is not addressed by ordinary benchmark accuracy. It asks for the size of a logical certificate. Proposition~\ref{prop:prefixlower} gives the general lower-bound principle, while Theorem~\ref{thm:linearcount} computes the obstruction in a finite linear family.

\begin{problem}[Prompt-stability]
Given prompts $p$ and $q$, decide whether every consequence of $p$ remains a consequence after appending $q$.
\end{problem}

Theorem~\ref{thm:preservation} gives an exact semantic criterion. It explains why prompt extension is not merely adding axioms: the hard content may increase while the selected preferred models change.

\begin{problem}[Minimal context extraction]
Given a finite deterministic task family, a current example set, and a query $q$, compute the smallest subcontext whose labels already force the answer to $q$.
\end{problem}

Theorem~\ref{thm:pairseparator} reduces this problem to a pair-separator hitting problem. Theorem~\ref{thm:certnp} shows that the resulting minimization problem is computationally hard even for binary outputs. This is the point at which finite-context semantics becomes an algorithmic object rather than a descriptive metaphor.

\begin{problem}[Semantic versus observed emergence]
Given confidence values $s_\lambda(\varphi)$ and a benchmark observation map, determine whether an apparent jump is forced by limit-theoretic convergence or by the discontinuity of the observation map.
\end{problem}

Section~\ref{sec:emergence} separates these cases. It also gives a rate-sensitive crossing bound, so the theory is not limited to qualitative convergence.

\subsection{Measurable definable events}

Let $\K\subseteq \Mod(\Lang)$ be a class of structures. For a sentence $\varphi\in\Sent(\Lang)$ write
\[
[\![\varphi]\!]_{\K}:=\{\M\in\K:\M\models \varphi\}.
\]
These are the elementary events visible to the semantic language.

\begin{definition}[Definable event algebra]
The Boolean algebra of sentence-definable events over $\K$ is
\[
\mathsf{Def}_{\Lang}(\K):=\{[\![\varphi]\!]_{\K}:\varphi\in\Sent(\Lang)\}.
\]
When a probability measure $\mu$ is defined on a $\sigma$-algebra $\A$ containing $\mathsf{Def}_{\Lang}(\K)$, the pair $(\K,\mu)$ is called a \emph{measured model class}.
\end{definition}

The Boolean operations are inherited from logic:
\[
[\![\neg\varphi]\!]_{\K}=\K\setminus[\![\varphi]\!]_{\K},
\qquad
[\![\varphi\wedge\psi]\!]_{\K}=[\![\varphi]\!]_{\K}\cap[\![\psi]\!]_{\K}.
\]
Thus probability on definable events is exactly probability on truth values.

\begin{lemma}[Elementary event algebra]\label{lem:boolean}
For every $\K\subseteq \Mod(\Lang)$, $\mathsf{Def}_{\Lang}(\K)$ is a Boolean subalgebra of $\mathcal{P}(\K)$. The map
\[
\varphi\longmapsto [\![\varphi]\!]_{\K}
\]
factors through logical equivalence over $\K$.
\end{lemma}

\begin{proof}
Closure under complement and finite intersection follows from negation and conjunction. Closure under finite union follows from disjunction. If $\varphi$ and $\psi$ are equivalent over $\K$, then for every $\M\in\K$ one has $\M\models\varphi$ exactly when $\M\models\psi$, so the corresponding events are equal.
\end{proof}

\begin{definition}[Semantic presentation]
Fix a scale parameter $\lambda\in\Lambda$, where $\Lambda$ is a directed set. A \emph{semantic presentation} at scale $\lambda$ is a tuple
\[
\mathbb{G}_{\lambda}=(\Lang,\mathcal{K}_{\lambda},\mu_{\lambda},\mathsf{C}_{\lambda},\mathsf{U}_{\lambda},\mathsf{D}_{\lambda})
\]
such that:
\begin{enumerate}[label=(\alph*),leftmargin=2.5em]
    \item $\Lang$ is a countable many-sorted first-order signature;
    \item $\mathcal{K}_{\lambda}\subseteq \Mod(\Lang)$ is a nonempty class of structures;
    \item $\mu_{\lambda}$ is a probability measure on a $\sigma$-algebra containing $\mathsf{Def}_{\Lang}(\mathcal{K}_{\lambda})$;
    \item $\mathsf{C}_{\lambda}$ is a nonempty set of context codes;
    \item $\mathsf{U}_{\lambda}$ maps $c\in\mathsf{C}_{\lambda}$ to a posterior measure
    \[
    \mu_{\lambda,c}:=\mathsf{U}_{\lambda}(c)(\mu_{\lambda})
    \]
    on the same measurable model class;
    \item $\mathsf{D}_{\lambda}$ is a decoding kernel such that $y\mapsto \mathsf{D}_{\lambda}(y\mid \M,c)$ is a probability distribution on outputs for every $\M\in\mathcal{K}_{\lambda}$ and $c\in\mathsf{C}_{\lambda}$.
\end{enumerate}
The induced observable distribution is
\[
\Prb_{\lambda}(y\mid c)=
\int_{\mathcal{K}_{\lambda}}\mathsf{D}_{\lambda}(y\mid \M,c)\,d\mu_{\lambda,c}(\M).
\]
\end{definition}

The definition separates three levels: model uncertainty, context update, and output decoding. A prompt changes the posterior semantic measure. Decoding then converts that posterior into tokens, strings, or structured outputs.

\begin{definition}[Semantic confidence and contextual theory]
For a sentence $\varphi\in\Sent(\Lang)$ define
\[
s_{\lambda,c}(\varphi):=\mu_{\lambda,c}([\![\varphi]\!]_{\mathcal{K}_{\lambda}}).
\]
The almost-sure contextual theory is
\[
\Th_{\lambda}(c):=\{\varphi\in\Sent(\Lang):s_{\lambda,c}(\varphi)=1\}.
\]
For a fragment $\mathcal{F}\subseteq\Sent(\Lang)$, write
\[
\Th_{\lambda}^{\mathcal F}(c):=\Th_{\lambda}(c)\cap \mathcal{F}.
\]
\end{definition}

The measure-one threshold is strict. It records semantic commitments that hold outside a null set of admissible structures. This is stronger than high probability and weaker than truth in every structure of $\mathcal{K}_{\lambda}$.

\begin{proposition}[Closure of contextual theories]\label{prop:contextclosure}
For every $\lambda$ and $c$, $\Th_{\lambda}(c)$ is closed under first-order consequence. More precisely, if $\Gamma\subseteq\Th_{\lambda}(c)$ is finite and $\Gamma\models\psi$, then $\psi\in\Th_{\lambda}(c)$.
\end{proposition}

\begin{proof}
Let $\Gamma=\{\gamma_1,\ldots,\gamma_m\}\subseteq\Th_{\lambda}(c)$. Each event $[\![\gamma_i]\!]$ has measure $1$. Therefore
\[
\mu_{\lambda,c}\left(\bigcap_{i=1}^m[\![\gamma_i]\!]\right)=1.
\]
If $\Gamma\models\psi$, then every structure satisfying all $\gamma_i$ also satisfies $\psi$, so
\[
\bigcap_{i=1}^m[\![\gamma_i]\!]\subseteq[\![\psi]\!].
\]
Thus $[\![\psi]\!]$ has measure $1$, and $\psi\in\Th_{\lambda}(c)$.
\end{proof}

\begin{proposition}[Consistency on realized support]\label{prop:consistency}
Assume $\mu_{\lambda,c}$ is countably additive. Every finite subset of $\Th_{\lambda}(c)$ is satisfiable in $\mathcal{K}_{\lambda}$. Hence $\Th_{\lambda}(c)$ is syntactically consistent.
\end{proposition}

\begin{proof}
If $\Delta=\{\delta_1,\ldots,\delta_m\}\subseteq\Th_{\lambda}(c)$, then the intersection
\[
A_{\Delta}:=\bigcap_{i=1}^m[\![\delta_i]\!]
\]
has measure $1$. In particular, it is nonempty. Any $\M\in A_{\Delta}$ satisfies $\Delta$. If $\Th_{\lambda}(c)$ were syntactically inconsistent, some finite subset would be inconsistent by compactness of proof systems, contradicting finite satisfiability.
\end{proof}

\begin{definition}[Decisive fragment at a context]
A fragment $\mathcal{F}\subseteq\Sent(\Lang)$ is \emph{decisive} at $(\lambda,c)$ if for every $\varphi\in\mathcal{F}$,
\[
s_{\lambda,c}(\varphi)\in\{0,1\}.
\]
\end{definition}

\begin{corollary}[Completeness on decisive fragments]\label{cor:decisivecomplete}
If $\mathcal{F}$ is closed under negation and decisive at $(\lambda,c)$, then for every $\varphi\in\mathcal{F}$ exactly one of $\varphi$ and $\neg\varphi$ belongs to $\Th_{\lambda}^{\mathcal F}(c)$.
\end{corollary}

\begin{proof}
For every sentence $\varphi$,
\[
s_{\lambda,c}(\neg\varphi)=1-s_{\lambda,c}(\varphi).
\]
If $s_{\lambda,c}(\varphi)$ is either $0$ or $1$, exactly one of $s_{\lambda,c}(\varphi)$ and $s_{\lambda,c}(\neg\varphi)$ is equal to $1$.
\end{proof}

\subsection{Output properties}

A generated output may be a token, string, proof sketch, tool call, or structured object. The semantic presentation handles this by putting outputs in a sort $\mathbf{Y}$ and using a decoding kernel.

\begin{definition}[Output property probability]
Let $\chi(y)$ be a formula with one free variable of output sort. Define the semantic-output probability
\[
P_{\lambda,c}(\chi):=
\int_{\mathcal{K}_{\lambda}}
\sum_{y:\,\M\models\chi(y)}
\mathsf{D}_{\lambda}(y\mid \M,c)\,d\mu_{\lambda,c}(\M),
\]
whenever the sum is measurable. If $\mathbf{Y}$ is finite, this is automatic.
\end{definition}

The value $s_{\lambda,c}(\exists y\,\chi(y))$ measures whether the semantic state admits an output satisfying $\chi$. The value $P_{\lambda,c}(\chi)$ measures whether the decoder actually emits such an output. The two are different. This separation is important for verification, because a property may be semantically entailed while a stochastic decoder still has residual failure probability.

\begin{definition}[Faithful decoding on a property]
A decoding kernel is \emph{$\epsilon$-faithful} to a sentence $\forall y(\alpha(y)\rightarrow\chi(y))$ at $(\lambda,c)$ if
\[
\int_{[\![\forall y(\alpha(y)\rightarrow\chi(y))]\!]}
\sum_{y:\,\M\models\alpha(y)\wedge\neg\chi(y)}
\mathsf{D}_{\lambda}(y\mid \M,c)\,d\mu_{\lambda,c}(\M)
\le \epsilon.
\]
\end{definition}

\begin{proposition}[Semantic entailment plus faithful decoding]\label{prop:faithful}
If $\forall y(\alpha(y)\rightarrow\chi(y))\in\Th_{\lambda}(c)$ and the decoder is $\epsilon$-faithful to that sentence at $(\lambda,c)$, then the probability of emitting an $\alpha$-output violating $\chi$ is at most $\epsilon$.
\end{proposition}

\begin{proof}
The violation probability is the integral of the decoder mass assigned to outputs satisfying $\alpha(y)\wedge\neg\chi(y)$. The set of structures satisfying $\forall y(\alpha(y)\rightarrow\chi(y))$ has measure $1$. The integral over its complement is therefore $0$, and the integral over the measure-one set is bounded by $\epsilon$ by definition.
\end{proof}

This proposition clarifies the verification target. Logic controls which structures are admissible. The decoder controls how reliably admissible semantic commitments are turned into observable strings.

\subsection{Boolean-algebraic semantics}

The same construction can be expressed through the Lindenbaum algebra of formulas modulo equivalence on the model class. This gives a compact way to see why almost-sure theories behave like filters and why decisive fragments behave like complete theories.

\begin{definition}[Lindenbaum algebra over a model class]
For $\varphi,\psi\in\Sent(\Lang)$ write
\[
\varphi\equiv_{\K}\psi
\quad\Longleftrightarrow\quad
[\![\varphi]\!]_{\K}=[\![\psi]\!]_{\K}.
\]
Let
\[
\mathbb{B}_{\K}:=\Sent(\Lang)/\equiv_{\K}
\]
with Boolean operations induced by $\neg,\wedge,\vee$. For a measured model class $(\K,\mu)$ define
\[
\bar\mu([\varphi]):=\mu([\![\varphi]\!]_{\K}).
\]
\end{definition}

\begin{proposition}[Semantic confidence as a Boolean probability measure]\label{prop:booleanmeasure}
The map $\bar\mu:\mathbb{B}_{\K}\to[0,1]$ is a finitely additive probability measure on the Boolean algebra $\mathbb{B}_{\K}$. That is,
\[
\bar\mu(\top)=1,
\qquad
\bar\mu(\bot)=0,
\]
and if $a\wedge b=\bot$, then
\[
\bar\mu(a\vee b)=\bar\mu(a)+\bar\mu(b).
\]
\end{proposition}

\begin{proof}
The value is well-defined because equivalent sentences have the same definable event. The identities for $\top$ and $\bot$ follow from $[\![\top]\!]_{\K}=\K$ and $[\![\bot]\!]_{\K}=\varnothing$. If $a=[\varphi]$ and $b=[\psi]$ are disjoint in the Boolean algebra, then $[\![\varphi]\!]_{\K}$ and $[\![\psi]\!]_{\K}$ are disjoint measurable sets. Finite additivity of $\mu$ gives the formula.
\end{proof}

\begin{definition}[Almost-sure filter]
For a measured model class $(\K,\mu)$ define
\[
\mathcal{F}_{\mu}:=\{[\varphi]\in\mathbb{B}_{\K}:\bar\mu([\varphi])=1\}.
\]
\end{definition}

\begin{theorem}[Measure-one formulas form a proper filter]\label{thm:filter}
If $\K\neq\varnothing$ and $\mu$ is a probability measure, then $\mathcal{F}_{\mu}$ is a proper filter in $\mathbb{B}_{\K}$. If, in addition, $\bar\mu(a)\in\{0,1\}$ for every $a\in\mathbb{B}_{\K}$, then $\mathcal{F}_{\mu}$ is an ultrafilter.
\end{theorem}

\begin{proof}
The top element has measure $1$, so it lies in $\mathcal{F}_{\mu}$. The bottom element has measure $0$, so the filter is proper. If $a,b\in\mathcal{F}_{\mu}$, then
\[
\bar\mu(a\wedge b)=1
\]
because the complement of $a\wedge b$ is contained in the union of the complements of $a$ and $b$, both of measure $0$. If $a\in\mathcal{F}_{\mu}$ and $a\le b$, then $\bar\mu(b)=1$. Thus $\mathcal{F}_{\mu}$ is a filter.

Assume now that all Boolean values have probability $0$ or $1$. For every $a$, either $\bar\mu(a)=1$ or $\bar\mu(\neg a)=1$, because $\bar\mu(\neg a)=1-\bar\mu(a)$. Hence the filter decides every Boolean element, so it is an ultrafilter.
\end{proof}

\begin{remark}
This theorem is the algebraic version of the contextual theory construction. The almost-sure theory is not merely a set of high-confidence statements. It is the syntactic face of a measure-one filter on definable events. When the semantic state is decisive, that filter becomes an ultrafilter, hence behaves like a complete possible-world point in the Stone space of the Boolean algebra.
\end{remark}

\subsection{Stone-space representation}

The Boolean-algebraic form has a canonical topological representation. Let $S_{\Lang}$ denote the Stone space of complete $\Lang$-theories. For a sentence $\varphi$, write
\[
\widehat{\varphi}:=\{p\in S_{\Lang}:\varphi\in p\}.
\]
The sets $\widehat{\varphi}$ are clopen and form a basis for the Stone topology.

\begin{definition}[Stone pushforward of a semantic presentation]
For a measured model class $(\K,\mu)$ define
\[
\pi_{\K}:\K\to S_{\Lang},
\qquad
\pi_{\K}(\M):=\Th(\M).
\]
The induced Stone measure is the pushforward
\[
\nu_{\K,\mu}:=(\pi_{\K})_{\ast}\mu.
\]
For a contextual posterior $\mu_{\lambda,c}$, write $\nu_{\lambda,c}$ for the corresponding Stone measure.
\end{definition}

\begin{theorem}[Stone representation of semantic confidence]\label{thm:stoneconfidence}
For every sentence $\varphi$,
\[
s_{\lambda,c}(\varphi)=\nu_{\lambda,c}(\widehat{\varphi}).
\]
Consequently, two contextual presentations induce the same semantic confidence function on all sentences if and only if they induce the same probability measure on the clopen algebra of $S_{\Lang}$.
\end{theorem}

\begin{proof}
By definition of $\pi_{\K}$,
\[
\pi_{\K}^{-1}(\widehat{\varphi})=
\{\M\in\K:\Th(\M)\in\widehat{\varphi}\}
=
\{\M\in\K:\M\models\varphi\}
=[\![\varphi]\!]_{\K}.
\]
Pushforward gives
\[
\nu_{\lambda,c}(\widehat{\varphi})
=
\mu_{\lambda,c}(\pi_{\K}^{-1}(\widehat{\varphi}))
=
\mu_{\lambda,c}([\![\varphi]\!]_{\K})
=s_{\lambda,c}(\varphi).
\]
If two Stone measures agree on all clopens, then the equality above gives identical confidences for all sentences. Conversely, equality of all confidences gives equality of the two measures on every basic clopen $\widehat{\varphi}$ and hence on the clopen algebra.
\end{proof}

\begin{corollary}[Elementary-equivalence invariance]\label{cor:elementaryinvariance}
Semantic confidence, contextual theory, threshold membership, and limit-theory membership depend only on the distribution of complete theories, not on the particular representatives in $\K$.
\end{corollary}

\begin{proof}
All four notions are functions of values $s_{\lambda,c}(\varphi)$. By Theorem~\ref{thm:stoneconfidence}, these values are exactly the Stone-measure values of clopens $\widehat{\varphi}$. Replacing structures by elementarily equivalent representatives does not change their image in $S_{\Lang}$, so it does not change the induced clopen probabilities.
\end{proof}

This representation is useful for later comparison across scales. A scale trend is a sequence of probability measures on a common Stone space, and a latent entailment is convergence of these measures to one on a clopen. This gives a representation-invariant target for semantic emergence.

\subsection{Support and indistinguishability}

A sentence can have confidence one without being true in every admissible structure. The right exact object is the support modulo null definable events.

\begin{definition}[Null equivalence]
For structures $\M,\Nstr\in\K$, define
\[
\M\equiv_{\mu}\Nstr
\]
when for every sentence $\varphi$, if $\M\models\varphi$ and $\Nstr\models\neg\varphi$, then neither $[\![\varphi]\!]_{\K}$ nor $[\![\neg\varphi]\!]_{\K}$ is forced by the measure-one filter. Equivalently, $\M$ and $\Nstr$ are not separated by any sentence whose truth value is decided almost surely.
\end{definition}

\begin{proposition}[Almost-sure theory is invariant under null indistinguishability]\label{prop:nullinvariant}
If $\M\equiv_{\mu}\Nstr$, then $\M$ and $\Nstr$ satisfy the same sentences in $\Th_{\mu}:=\{\varphi:\mu([\![\varphi]\!]_{\K})=1\}$.
\end{proposition}

\begin{proof}
Suppose $\varphi\in\Th_{\mu}$ and $\M\models\varphi$. If $\Nstr\models\neg\varphi$, then $\varphi$ separates $\M$ and $\Nstr$ by a sentence decided almost surely, contradicting $\M\equiv_{\mu}\Nstr$. Hence $\Nstr\models\varphi$. The reverse direction is identical.
\end{proof}

This quotient view is useful because LLM behavior is rarely determined by one canonical structure. Many latent structures may be behaviorally indistinguishable under the formulas and contexts being tested.

\subsection{Definitional invariance}

A semantic theory for LLM behavior should not depend on accidental notation. If the language is expanded only by definitional abbreviations, the semantic confidence of old-language claims should not change. This is the first representation-invariance test for the construction.

\begin{definition}[Conservative definitional expansion]
Let $\Lang\subseteq\Lang'$ and let $T'$ be a theory in $\Lang'$. We say that $T'$ is a conservative definitional expansion of $T$ over a model class $\K$ if every $\M\in\K$ has a unique expansion $\M'\models T'$ to $\Lang'$ and, for every $\Lang$-sentence $\varphi$,
\[
\M\models\varphi
\qquad\Longleftrightarrow\qquad
\M'\models\varphi.
\]
Let $e:\K\to\K'$ send $\M$ to its unique expansion.
\end{definition}

\begin{theorem}[Invariance under definitional expansion]\label{thm:definvinv}
Let $(\K,\mu)$ be a measured $\Lang$-model class and let $(\K',\mu')$ be obtained by a conservative definitional expansion, with $\mu'=e_\ast\mu$. Then for every $\Lang$-sentence $\varphi$,
\[
\mu([\![\varphi]\!]_\K)=\mu'([\![\varphi]\!]_{\K'}).
\]
Consequently, almost-sure theories, prompt consequences expressed in $\Lang$, and threshold membership for $\Lang$-sentences are invariant under such expansions.
\end{theorem}

\begin{proof}
By the definition of conservative expansion,
\[
e^{-1}([\![\varphi]\!]_{\K'})=[\![\varphi]\!]_\K
\]
for every old-language sentence $\varphi$. Since $\mu'=e_\ast\mu$,
\[
\mu'([\![\varphi]\!]_{\K'})=\mu(e^{-1}([\![\varphi]\!]_{\K'}))=\mu([\![\varphi]\!]_\K).
\]
The conclusions about almost-sure theories and threshold membership follow directly from equality of confidence values. Prompt consequence is defined by satisfaction of the same old-language sentences on corresponding selected structures, so it is also preserved.
\end{proof}

\section{Prompts as preferential context updates}\label{sec:prompt}

A prompt has literal content and pragmatic force. Literal content behaves like constraints. Pragmatic force ranks models that satisfy those constraints. This is the same structural division that appears in default logic and preferential consequence, although the object being updated here is a measured class of semantic structures. Reiter's default logic gives one classical formalization of defaults \cite{Reiter1980}. Preferential semantics and cumulative consequence relations give another \cite{Kraus1990}. Belief revision separates hard information from revision policy in a related way \cite{AGM1985}.

\begin{definition}[Prompt specification]
Let $X\subseteq\mathcal{K}_{\lambda}$. A \emph{prompt specification} on $X$ is a pair
\[
p=(H_p,\rho_p),
\]
where $H_p\subseteq\Sent(\Lang)$ is finite and $\rho_p:X\to\Gamma_p$ maps structures into a well-ordered set. Lower rank means greater pragmatic preference. Define
\[
\mathsf{Upd}_p(X):=\Min_{\rho_p}(X\cap\Mod(H_p)),
\]
where
\[
\Min_{\rho_p}(A):=\{\M\in A:(\forall\Nstr\in A)\ \rho_p(\M)\le\rho_p(\Nstr)\}.
\]
If $\mathsf{Upd}_p(X)\neq\varnothing$, define prompt consequence by
\[
p\Vdash_X\varphi
\quad\Longleftrightarrow\quad
(\forall\M\in\mathsf{Upd}_p(X))\ \M\models\varphi.
\]
\end{definition}

\begin{proposition}[Elementary algebra of prompt update]\label{prop:promptalgebra}
Assume $X\cap\Mod(H_p)\neq\varnothing$. Then:
\begin{enumerate}[label=(\roman*),leftmargin=2.5em]
    \item $\mathsf{Upd}_p(X)\subseteq X\cap\Mod(H_p)$;
    \item $\mathsf{Upd}_p(\mathsf{Upd}_p(X))=\mathsf{Upd}_p(X)$;
    \item if $\rho_p$ is constant on $X\cap\Mod(H_p)$, then
    \[
    \mathsf{Upd}_p(X)=X\cap\Mod(H_p).
    \]
\end{enumerate}
\end{proposition}

\begin{proof}
The first item is immediate from the definition. For the second, put $A=X\cap\Mod(H_p)$. All elements of $\Min_{\rho_p}(A)$ have the same minimal rank in $A$. Applying $\Min_{\rho_p}$ again to that set therefore changes nothing. For the third, if the rank is constant on $A$, every element of $A$ is minimal.
\end{proof}

The constant-rank case recovers literal context restriction. The general case allows an instruction to prefer one interpretation over another even when both satisfy the literal constraints.

\begin{definition}[Lexicographic composition]
Given prompts $p=(H_p,\rho_p)$ and $q=(H_q,\rho_q)$ on $X$, their lexicographic composition is
\[
p\oplus q=(H_p\cup H_q,\rho_{p\oplus q}),
\]
where
\[
\rho_{p\oplus q}(\M):=(\rho_q(\M),\rho_p(\M))
\]
ordered lexicographically. Thus later pragmatic force has priority, while hard constraints accumulate.
\end{definition}

Other priority conventions are possible. The lexicographic convention is useful because it captures a common instruction-following pattern: a later style or role instruction may dominate an earlier default, provided the hard content remains satisfiable.

\begin{theorem}[Prompt consequence is nonmonotonic]\label{thm:nonmon}
Suppose $X$ contains $\M$ and $\Nstr$, and there are prompts $p,q$ and a sentence $\varphi$ such that:
\begin{enumerate}[label=(\alph*),leftmargin=2.5em]
    \item $\M,\Nstr\in X\cap\Mod(H_p\cup H_q)$;
    \item $\mathsf{Upd}_p(X)=\{\M\}$;
    \item $\mathsf{Upd}_{p\oplus q}(X)=\{\Nstr\}$;
    \item $\M\models\varphi$ and $\Nstr\models\neg\varphi$.
\end{enumerate}
Then
\[
p\Vdash_X\varphi
\qquad\text{and}\qquad
p\oplus q\not\Vdash_X\varphi.
\]
\end{theorem}

\begin{proof}
Since $\mathsf{Upd}_p(X)=\{\M\}$ and $\M\models\varphi$, one has $p\Vdash_X\varphi$. Since $\mathsf{Upd}_{p\oplus q}(X)=\{\Nstr\}$ and $\Nstr\models\neg\varphi$, not every structure selected by $p\oplus q$ satisfies $\varphi$. Hence $p\oplus q\not\Vdash_X\varphi$.
\end{proof}

This is not ordinary monotone theory extension. Hard information can be monotone while pragmatic preference is not.

\subsection{Fixed-preference closure}

Nonmonotonicity does not mean absence of structure. If the ranking is fixed and only hard constraints vary, prompt consequence satisfies controlled closure laws.

\begin{definition}[Ranked consequence under fixed preference]
Fix $X$ and a ranking $\rho:X\to\Gamma$. For a finite $H\subseteq\Sent(\Lang)$ define
\[
H\Vdash_{X,\rho}\varphi
\quad\Longleftrightarrow\quad
\Min_{\rho}(X\cap\Mod(H))\subseteq[\![\varphi]\!]_X.
\]
\end{definition}

\begin{proposition}[Right weakening and conjunction]\label{prop:rightweak}
For fixed $X$ and $\rho$:
\begin{enumerate}[label=(\roman*),leftmargin=2.5em]
    \item if $H\Vdash_{X,\rho}\varphi$ and $\varphi\models\psi$, then $H\Vdash_{X,\rho}\psi$;
    \item if $H\Vdash_{X,\rho}\varphi$ and $H\Vdash_{X,\rho}\psi$, then $H\Vdash_{X,\rho}\varphi\wedge\psi$.
\end{enumerate}
\end{proposition}

\begin{proof}
Both claims follow by inclusion. In (i), $[\![\varphi]\!]_X\subseteq[\![\psi]\!]_X$. In (ii), the selected set is included in both $[\![\varphi]\!]_X$ and $[\![\psi]\!]_X$, hence in their intersection.
\end{proof}

\begin{theorem}[Cautious monotony under stable minima]\label{thm:cautious}
Let $H$ be finite. If
\[
H\Vdash_{X,\rho}\varphi
\]
and
\[
\Min_{\rho}(X\cap\Mod(H\cup\{\varphi\}))=\Min_{\rho}(X\cap\Mod(H)),
\]
then for every $\psi$,
\[
H\Vdash_{X,\rho}\psi
\quad\Longleftrightarrow\quad
H\cup\{\varphi\}\Vdash_{X,\rho}\psi.
\]
\end{theorem}

\begin{proof}
Under the stated equality, the selected model set for $H$ is exactly the selected model set for $H\cup\{\varphi\}$. The two consequence relations therefore quantify over the same structures.
\end{proof}

The hypothesis says that adding a consequence already true in all preferred models does not disturb the set of preferred models. In prompt terms, a restatement of what the current preferred interpretation already satisfies is harmless. A later instruction is dangerous only when it changes the preference order or removes all current minima.

\begin{definition}[Theory of a selected set]
For $S\subseteq X$, define
\[
\Th_X(S):=\{\varphi\in\Sent(\Lang):S\subseteq[\![\varphi]\!]_X\}.
\]
\end{definition}

\begin{theorem}[Exact preservation criterion for prompt extension]\label{thm:preservation}
Assume $\mathsf{Upd}_p(X)$ and $\mathsf{Upd}_{p\oplus q}(X)$ are nonempty. Then the following are equivalent:
\begin{enumerate}[label=(\roman*),leftmargin=2.5em]
    \item every $p$-consequence is preserved after appending $q$, that is,
    \[
    p\Vdash_X\varphi\quad\Longrightarrow\quad p\oplus q\Vdash_X\varphi
    \]
    for all sentences $\varphi$;
    \item
    \[
    \Th_X(\mathsf{Upd}_p(X))\subseteq \Th_X(\mathsf{Upd}_{p\oplus q}(X));
    \]
    \item every sentence true in all $p$-selected structures is true in all $(p\oplus q)$-selected structures.
\end{enumerate}
If $\mathsf{Upd}_p(X)$ is finite and the language separates points of $X$, then these equivalent conditions imply
\[
\mathsf{Upd}_{p\oplus q}(X)\subseteq \mathsf{Upd}_p(X).
\]
\end{theorem}

\begin{proof}
The equivalence of (i), (ii), and (iii) is just the definition of skeptical consequence over selected sets. For the final claim, suppose $\mathsf{Upd}_p(X)=\{\M_1,\ldots,\M_m\}$ and take $\Nstr\in\mathsf{Upd}_{p\oplus q}(X)$. If $\Nstr\notin\mathsf{Upd}_p(X)$, then for each $i$ there is a sentence $\theta_i$ with $\M_i\models\theta_i$ and $\Nstr\models\neg\theta_i$, after negating the separator if necessary. The conjunction $\theta=\bigwedge_i\theta_i$ is true on all $p$-selected structures and false at $\Nstr$. Hence $\theta\in\Th_X(\mathsf{Upd}_p(X))$ but $\theta\notin\Th_X(\mathsf{Upd}_{p\oplus q}(X))$, contradicting (iii). Thus every $(p\oplus q)$-selected structure lies in $\mathsf{Upd}_p(X)$.
\end{proof}

This theorem turns a vague instruction-following question into a semantic test. A new instruction is safe exactly when it does not move the selected structures outside the old theory. If the relevant selected set is finite and the language distinguishes its points, safety is equivalent to containment of selected sets.

\subsection{Prompt update as posterior restriction}

The previous definitions are set-theoretic. They connect to the measured presentation by conditioning and preference weighting.

\begin{definition}[Soft preferential posterior]
Let $p=(H_p,\rho_p)$ be a prompt and let $\beta>0$. Define the unnormalized weight
\[
w_p^{\beta}(\M):=\mathbf{1}_{\M\models H_p}\exp(-\beta r_p(\M)),
\]
where $r_p:X\to\R_{\ge 0}$ is a real-valued representation of the ranking. If
\[
Z_p^{\beta}:=\int_X w_p^{\beta}(\M)\,d\mu(\M)>0,
\]
then the soft prompt posterior is
\[
d\mu_p^{\beta}(\M):=\frac{w_p^{\beta}(\M)}{Z_p^{\beta}}\,d\mu(\M).
\]
\end{definition}

\begin{proposition}[Zero-temperature limit]\label{prop:zerotemp}
Assume $X$ is finite and $r_p$ attains its minimum over $X\cap\Mod(H_p)$. Let
\[
A_0:=\argmin\{r_p(\M):\M\in X\cap\Mod(H_p)\}.
\]
If $\mu(\M)>0$ for every $\M\in X$, then for every $B\subseteq X$,
\[
\lim_{\beta\to\infty}\mu_p^{\beta}(B)=\frac{\mu(B\cap A_0)}{\mu(A_0)}.
\]
\end{proposition}

\begin{proof}
Let $m$ be the minimum of $r_p$ on $X\cap\Mod(H_p)$. Then
\[
Z_p^{\beta}=e^{-\beta m}\mu(A_0)+
\sum_{\M\in (X\cap\Mod(H_p))\setminus A_0}e^{-\beta r_p(\M)}\mu(\M).
\]
After dividing numerator and denominator by $e^{-\beta m}$, every term with $r_p(\M)>m$ tends to $0$. The remaining mass is exactly the prior mass restricted to $A_0$ and normalized.
\end{proof}

Thus hard preferential update is the zero-temperature limit of a familiar probabilistic selection rule.

\subsection{Logical cost models for prompt conflict}

Prompt instructions may be mutually satisfiable as hard constraints while still competing pragmatically. A useful formal abstraction is to assign costs to violations and then recover preferred models by minimization.

\begin{definition}[Violation cost]\label{def:violationcost}
Let $H=\{\eta_1,\ldots,\eta_m\}$ be a finite instruction set and let $w_i>0$. The violation cost of a structure $\M\in X$ is
\[
\ell_H(\M):=\sum_{i:\,\M\models\neg\eta_i}w_i.
\]
The induced update is
\[
\mathsf{CostUpd}_H(X):=\argmin_{\M\in X}\ell_H(\M).
\]
\end{definition}

\begin{proposition}[Hard constraints as infinite-cost limit]\label{prop:hardlimit}
Let $H=H_0\cup H_1$, where $H_0$ is treated as hard and $H_1$ as soft. Assign weight $B$ to every formula in $H_0$ and fixed finite positive weights to formulas in $H_1$. If $X\cap\Mod(H_0)\neq\varnothing$, then for all sufficiently large $B$ every element of $\mathsf{CostUpd}_H(X)$ satisfies $H_0$.
\end{proposition}

\begin{proof}
Because $H_1$ is finite and has fixed weights, the maximum possible soft cost is some finite number $C$. Any model satisfying $H_0$ has hard cost $0$ and total cost at most $C$. Any model violating at least one hard formula has hard cost at least $B$. For $B>C$, no hard-violating model can be cost-minimal.
\end{proof}

\begin{theorem}[Cost update agrees with preferential update]\label{thm:costpref}
Assume $X$ is finite. For every real-valued ranking $r:X\to\R$ there exists a cost function $\ell:X\to\R_{\ge0}$ such that
\[
\argmin_X r=\argmin_X \ell.
\]
Conversely, every cost update is a preferential update for the ranking $\rho(\M)=\ell(\M)$.
\end{theorem}

\begin{proof}
For the first direction, choose $m=\min_X r$ and set $\ell(\M)=r(\M)-m$. Then $\ell\ge0$ and the minimizers are exactly the minimizers of $r$. The converse is the definition of preferential update with rank equal to cost.
\end{proof}

This shows that weighted-instruction semantics and ranked-model semantics are two presentations of the same selection principle on finite model classes.

\section{In-context learning as semantic model expansion}\label{sec:icl}

In-context learning can be formalized without assuming that parameters change. The semantic state changes because the context extends the language with observed examples and adds the finite diagram of those examples. The direct-limit theory then represents what the entire example stream would determine.

\begin{definition}[Example-driven expansion chain]
Let $\Lang_0$ contain sorts $\mathbf{X}$ and $\mathbf{Y}$ and a function symbol $f:\mathbf{X}\to\mathbf{Y}$. Let $T_0\subseteq\Sent(\Lang_0)$ be consistent. Given an example stream
\[
E=(e_n)_{n\in\N}=((a_n,b_n))_{n\in\N},
\]
where $a_n$ and $b_n$ are represented by closed terms of the appropriate sorts, define recursively
\begin{align*}
\Lang_n&:=\Lang_{n-1}\cup\{c_n,d_n\},\\
\Delta_n&:=\{c_n=a_n,\ d_n=b_n,\ f(c_n)=d_n\},\\
T_n&:=\Cn_{\Lang_n}(T_{n-1}^{\up}\cup\Delta_n),\\
\mathcal{K}_n&:=\Mod(T_n).
\end{align*}
The chain $(\Lang_n,T_n,\mathcal{K}_n)_{n\in\N}$ is the semantic expansion chain generated by $E$.
\end{definition}

\begin{proposition}[Monotone expansion of information]\label{prop:monotoneT}
For every $n$,
\[
T_n^{\up}\subseteq T_{n+1}.
\]
Consequently, if $\M\models T_{n+1}$, then the $\Lang_n$-reduct $\M\restr\Lang_n$ is a model of $T_n$.
\end{proposition}

\begin{proof}
By construction,
\[
T_{n+1}=\Cn_{\Lang_{n+1}}(T_n^{\up}\cup\Delta_{n+1}).
\]
A consequence set contains its premises. The reduct statement follows because the $\Lang_n$-sentences true in $\M$ are exactly the $\Lang_n$-sentences true in the reduct.
\end{proof}

\begin{definition}[Admissible answer sets]
Let $q$ be a closed $\Lang_0$-term of sort $\mathbf{X}$ and let $\mathcal{C}_{\mathbf Y}$ be a fixed set of closed terms of sort $\mathbf{Y}$. Define
\[
\mathrm{Ans}_n(q):=\{t\in\mathcal{C}_{\mathbf Y}:T_n\cup\{f(q)=t\}\text{ is satisfiable}\}.
\]
\end{definition}

\begin{proposition}[Monotone narrowing of admissible answers]\label{prop:narrowing}
For every query $q$ and every $n$,
\[
\mathrm{Ans}_{n+1}(q)\subseteq\mathrm{Ans}_n(q).
\]
\end{proposition}

\begin{proof}
If $t\in\mathrm{Ans}_{n+1}(q)$, choose $\M\models T_{n+1}\cup\{f(q)=t\}$. By Proposition~\ref{prop:monotoneT}, the $\Lang_n$-reduct of $\M$ satisfies $T_n$. Since $f,q,t$ are already in $\Lang_n$, the reduct also satisfies $f(q)=t$. Hence $t\in\mathrm{Ans}_n(q)$.
\end{proof}

\subsection{Types and answer isolation}

The shrinking of answer sets can be restated as type isolation. This formulation is useful because it expresses what examples do semantically: they reduce the space of complete types compatible with the context.

\begin{definition}[Contextual query type]
Let $x$ be a variable of sort $\mathbf{X}$. The $n$-stage type of a query term $q$ is
\[
\tp_n(q):=\{\theta(x)\in\Lang_n: T_n\models \theta(q)\}.
\]
For an answer term $t$, say that $\tp_n(q)$ \emph{isolates} $t$ if
\[
T_n\models f(q)=t.
\]
\end{definition}

\begin{lemma}[Isolation equals singleton admissibility]\label{lem:singleton}
Assume distinct terms in $\mathcal{C}_{\mathbf Y}$ denote distinct candidate answers in all models of $T_n$. Then $\tp_n(q)$ isolates $t$ if and only if
\[
\mathrm{Ans}_n(q)=\{t\}.
\]
\end{lemma}

\begin{proof}
If $T_n\models f(q)=t$, then any satisfiable extension $T_n\cup\{f(q)=u\}$ forces $u=t$ by distinctness of candidates. Hence the answer set is $\{t\}$, provided it is nonempty. Nonemptiness follows from consistency of $T_n$ and interpretation of $f(q)$ among the candidate terms. Conversely, if $\mathrm{Ans}_n(q)=\{t\}$ and $T_n\not\models f(q)=t$, then $T_n\cup\{f(q)\neq t\}$ is satisfiable. In such a model $f(q)$ must be denoted by some candidate $u\neq t$, giving $u\in\mathrm{Ans}_n(q)$, a contradiction.
\end{proof}

\begin{theorem}[Finite stabilization by compactness]\label{thm:stabilization}
Let
\[
T_{\omega}:=\bigcup_{n\in\N}T_n.
\]
If $q$ is a query term and
\[
T_{\omega}\models f(q)=t,
\]
then there exists $N\in\N$ such that
\[
T_N\models f(q)=t.
\]
In particular, if the direct-limit theory determines a unique answer to $q$, some finite prefix of the context already determines it.
\end{theorem}

\begin{proof}
The entailment $T_{\omega}\models f(q)=t$ means that
\[
T_{\omega}\cup\{f(q)\neq t\}
\]
is inconsistent. By compactness, some finite $\Sigma\subseteq T_{\omega}$ already makes
\[
\Sigma\cup\{f(q)\neq t\}
\]
inconsistent. Since the chain is increasing, there exists $N$ with $\Sigma\subseteq T_N$. Thus $T_N\cup\{f(q)\neq t\}$ is inconsistent, equivalently $T_N\models f(q)=t$.
\end{proof}

The theorem is not a claim that every finite prompt determines every answer. It says that any answer determined by the infinite semantic expansion has a finite certificate. This is the compactness core behind finite in-context evidence.

\begin{corollary}[Finite certificate for answer exclusion]\label{cor:exclusion}
If $t\notin\mathrm{Ans}_{\omega}(q)$, where
\[
\mathrm{Ans}_{\omega}(q):=\{u:T_{\omega}\cup\{f(q)=u\}\text{ is satisfiable}\},
\]
then there exists $N$ such that $t\notin\mathrm{Ans}_N(q)$.
\end{corollary}

\begin{proof}
The statement $t\notin\mathrm{Ans}_{\omega}(q)$ means $T_{\omega}\models f(q)\neq t$. Apply Theorem~\ref{thm:stabilization} to the formula $f(q)\neq t$.
\end{proof}

\begin{proposition}[Indistinguishable-prefix lower bound]\label{prop:prefixlower}
Let $\chi(f(q))$ be a target property. If there exist two models $\M,\Nstr\models T_n$ such that
\[
\M\models\chi(f(q))
\qquad\text{and}\qquad
\Nstr\models\neg\chi(f(q)),
\]
then no prefix of length at most $n$ certifies the target. In particular, the least certificate length is strictly larger than $n$.
\end{proposition}

\begin{proof}
If some $T_m$ with $m\le n$ entailed $\chi(f(q))$, then $T_n$ would also entail $\chi(f(q))$ by monotonicity of the expansion chain. But $\Nstr\models T_n\cup\{\neg\chi(f(q))\}$, contradicting entailment. Hence no such $m$ exists.
\end{proof}

This lower bound is the model-theoretic form of underdetermination in in-context learning. To prove that a context is too short, it is enough to build two extensions of the same finite prefix that agree with all observed examples but disagree on the queried property.

\subsection{Pair-separator certificates and teaching dimension}

Compactness proves that first-order certificates are finite when they exist, but it does not say which examples are responsible. In finite deterministic task families the responsible examples have an exact combinatorial description.

\begin{definition}[Finite deterministic task family]
Let $H$ be a finite set of hypotheses, let $X$ be a finite set of possible example inputs, let $B$ be a finite output alphabet, and let
\[
h:X\cup\{q\}\to B
\]
be the response function associated with each $h\in H$. After observing a labeled example set $E\subseteq X$, the version space is
\[
H_E:=\{h\in H:(\forall x\in E)\ h(x)=b_x\},
\]
where $b_x$ is the observed label. The query $q$ is \emph{determined by $E$} if all hypotheses in $H_E$ have the same value at $q$.
\end{definition}

For a fixed realized version space $V\subseteq H$, define the query-disagreement pairs
\[
P_q(V):=\{\{h,h'\}\subseteq V:h(q)\neq h'(q)\}.
\]
For $x\in X$, let
\[
\Sep_x(V):=\{\{h,h'\}\in P_q(V):h(x)\neq h'(x)\}.
\]
Thus $x$ separates precisely those pairs of surviving hypotheses that would answer the query differently.

\begin{theorem}[Exact pair-separator characterization]\label{thm:pairseparator}
Let $V$ be the version space after a fixed background context. A set $E\subseteq X$ determines $q$ relative to $V$ if and only if
\[
P_q(V)\subseteq \bigcup_{x\in E}\Sep_x(V).
\]
Equivalently, finite-context determination of $q$ is exactly a hitting problem over query-disagreement pairs.
\end{theorem}

\begin{proof}
Assume first that $E$ determines $q$. Let $\{h,h'\}\in P_q(V)$. Since $h(q)\neq h'(q)$, the two hypotheses cannot both remain compatible with the labels on $E$. Therefore there exists some $x\in E$ such that $h(x)\neq h'(x)$, so $\{h,h'\}\in\Sep_x(V)$.

Conversely, suppose every pair in $P_q(V)$ is separated by some $x\in E$. Let $h,h'\in V$ agree on all labels in $E$. Then $\{h,h'\}$ cannot lie in any $\Sep_x(V)$ with $x\in E$. By the assumed covering condition, it cannot belong to $P_q(V)$. Hence $h(q)=h'(q)$. All hypotheses compatible with $E$ therefore agree on $q$, so $E$ determines $q$.
\end{proof}

\begin{definition}[Query teaching dimension]
The \emph{query teaching dimension} of $q$ relative to $V$ is
\[
\TD_q(V):=\min\{|E|:E\subseteq X\text{ determines }q\text{ relative to }V\}.
\]
If no such $E$ exists, set $\TD_q(V)=\infty$.
\end{definition}

Theorem~\ref{thm:pairseparator} says that $\TD_q(V)$ is the minimum size of a set of examples whose separator sets cover $P_q(V)$. This turns semantic context selection into an exact finite optimization problem.

\begin{theorem}[Minimal certificate extraction is NP-complete]\label{thm:certnp}
The decision problem
\[
\TD_q(V)\le k
\]
is NP-complete even when the output alphabet is binary.
\end{theorem}

\begin{proof}
Membership in NP is immediate: a candidate set $E$ of at most $k$ examples can be checked by verifying the covering condition in Theorem~\ref{thm:pairseparator}.

For hardness, reduce Set Cover to the certificate problem. Let a Set Cover instance have universe $U=\{u_1,\ldots,u_m\}$, subsets $S_1,\ldots,S_r\subseteq U$, and budget $k$. Create examples $x_1,\ldots,x_r$, one for each subset, and a query $q$. For every element $u\in U$, create two hypotheses $a_u$ and $b_u$. Define binary labels by
\[
a_u(q)=0,
\qquad
b_u(q)=1.
\]
For example $x_i$, set
\[
a_u(x_i)=0,
\qquad
b_u(x_i)=
\begin{cases}
1,&u\in S_i,\\
0,&u\notin S_i.
\end{cases}
\]
Let $V:=\{a_u,b_u:u\in U\}$. The pair $\{a_u,b_u\}$ is separated by $x_i$ exactly when $u\in S_i$. If a collection of examples $E$ determines $q$, then in particular it separates every pair $\{a_u,b_u\}$, so the corresponding subsets cover every element of $U$. Conversely, if subsets indexed by $I$ cover $U$, then the examples $\{x_i:i\in I\}$ separate every pair $\{a_v,b_u\}$ with different query value, because $u$ is covered by some selected $S_i$ and then $b_u(x_i)=1$ while every $a_v(x_i)=0$. Hence they determine $q$. Therefore a size-$k$ context certificate exists exactly when the Set Cover instance has a size-$k$ cover. Since Set Cover is NP-complete \cite{Karp1972}, the certificate decision problem is NP-complete.
\end{proof}

This theorem is the main finite-context obstruction. Even if the semantic presentation is exact and the hypothesis class is finite, extracting the smallest explanatory context is computationally hard. Thus benchmark success does not automatically yield a short human-checkable certificate; the certificate itself is a nontrivial combinatorial object.

\subsection{Conservative expansion and genuine information}

Adding names to the language does not itself add information about old symbols. The information comes from diagrams and constraints.

\begin{definition}[Conservative context step]
The step $T_n\subseteq T_{n+1}$ is \emph{conservative over $\Lang_n$} if every $\Lang_n$-sentence $\varphi$ satisfying
\[
T_{n+1}\models\varphi
\]
already satisfies
\[
T_n\models\varphi.
\]
\end{definition}

\begin{proposition}[When examples change old consequences]\label{prop:nonconservative}
The step from $T_n$ to $T_{n+1}$ is nonconservative over $\Lang_n$ exactly when there exists an old-language sentence $\varphi\in\Sent(\Lang_n)$ such that
\[
T_n\cup\{\neg\varphi\}\text{ is satisfiable}
\qquad\text{but}\qquad
T_{n+1}\models\varphi.
\]
\end{proposition}

\begin{proof}
This is the negation of conservativity written out. If the extension is nonconservative, there is an old sentence entailed by the extension but not by the base. Not being entailed by the base is equivalent to satisfiability of $T_n\cup\{\neg\varphi\}$. The converse is immediate.
\end{proof}

Thus an example is semantically informative about prior vocabulary exactly when it eliminates at least one old-language possibility.

\subsection{No purely infinite first-order answer jumps}

Compactness also gives a negative result: first-order entailments cannot appear only at the infinite limit while being absent from every finite stage.

\begin{theorem}[No infinite-only first-order certificate]\label{thm:noinfiniteonly}
Let $\psi$ be a sentence in the union language $\bigcup_n\Lang_n$. If
\[
T_n\not\models\psi
\qquad\text{for every }n,
\]
then
\[
T_{\omega}\not\models\psi.
\]
Equivalently, if $T_{\omega}\models\psi$, then $T_N\models\psi$ for some finite $N$.
\end{theorem}

\begin{proof}
The second statement is the compactness argument used in Theorem~\ref{thm:stabilization}. For the first statement, assume $T_n\not\models\psi$ for every $n$. Then $T_n\cup\{\neg\psi\}$ is satisfiable for every $n$ large enough to contain the language of $\psi$. Let $\Sigma$ be any finite subset of $T_{\omega}$. Since the chain is increasing, $\Sigma\subseteq T_N$ for some $N$. Then $T_N\cup\{\neg\psi\}$ is satisfiable, so $\Sigma\cup\{\neg\psi\}$ is satisfiable. Every finite subset of $T_{\omega}\cup\{\neg\psi\}$ is satisfiable. By compactness, $T_{\omega}\cup\{\neg\psi\}$ is satisfiable, hence $T_{\omega}\not\models\psi$.
\end{proof}

The theorem is a useful guardrail. If a claimed in-context phenomenon is expressible in first-order form and is entailed by the entire context stream, then it has a finite logical certificate. If no finite prefix can certify it, then the claimed property is either not first-order in this presentation, not an entailment, or depends on an external limiting operation.

\subsection{Quantifier rank and finite evidence profiles}

The previous compactness statements do not quantify the size of the finite certificate. A finer analysis tracks formulas by syntactic complexity.

\begin{definition}[Rank-restricted consequence]
Let $\Sent_{\le r}(\Lang)$ be the set of sentences of quantifier rank at most $r$. Define
\[
T_n^{\le r}:=T_n\cap\Sent_{\le r}(\Lang_n).
\]
For a sentence $\psi$, write
\[
T_n\models_r\psi
\]
if there is a finite $\Sigma\subseteq T_n^{\le r}$ such that $\Sigma\models\psi$.
\end{definition}

\begin{proposition}[Rank-restricted monotonicity]\label{prop:rankmono}
If $T_n\models_r\psi$, then $T_m\models_r\psi$ for every $m\ge n$, after viewing formulas in the larger language.
\end{proposition}

\begin{proof}
Choose finite $\Sigma\subseteq T_n^{\le r}$ with $\Sigma\models\psi$. Since $T_n^{\up}\subseteq T_m$, the same $\Sigma$ is contained in $T_m$. Quantifier rank is unchanged by viewing a sentence in a larger language.
\end{proof}

\begin{definition}[Evidence profile]
For a property $\psi$, define its evidence profile along the context chain by
\[
\mathcal{E}_{\psi}:=\{(n,r):T_n\models_r\psi\}.
\]
\end{definition}

\begin{proposition}[Upward closure of evidence profiles]\label{prop:evidenceup}
If $(n,r)\in\mathcal{E}_{\psi}$, then $(m,s)\in\mathcal{E}_{\psi}$ for every $m\ge n$ and $s\ge r$.
\end{proposition}

\begin{proof}
The condition $m\ge n$ is handled by Proposition~\ref{prop:rankmono}. The condition $s\ge r$ holds because $\Sent_{\le r}(\Lang_m)\subseteq\Sent_{\le s}(\Lang_m)$.
\end{proof}

Thus a certified property has an upward-closed region in the two-dimensional plane of context length and logical complexity. This is a precise way to ask whether a behavior needs more examples, more expressive formulas, or both.

\section{Latent task families and exact identification}\label{sec:latenttasks}

The preceding section is abstract. This section records two concrete identification mechanisms. The first is linear and algebraic. The second is purely definability-based.

\subsection{Linear task family over a finite field}

Let $\F$ be a finite field and $d\ge 1$. Let $\mathbf{X}=\F^d$ and $\mathbf{Y}=\F$. Suppose the background theory asserts the existence of a latent parameter $w\in\F^d$ such that
\[
f(x)=w^{\top}x
\qquad\text{for all }x\in\F^d.
\]
Given examples $(a_i,b_i)$ for $1\le i\le n$, define
\[
A_n=\begin{pmatrix}
a_1^{\top}\\
\vdots\\
a_n^{\top}
\end{pmatrix}\in\F^{n\times d},
\qquad
b^{(n)}=\begin{pmatrix}
b_1\\
\vdots\\
b_n
\end{pmatrix}\in\F^n,
\]
and
\[
W_n:=\{w\in\F^d:A_nw=b^{(n)}\}.
\]

\begin{theorem}[Exact identifiability in the linear case]\label{thm:linear}
Assume $W_n\neq\varnothing$. Then:
\begin{enumerate}[label=(\roman*),leftmargin=2.5em]
    \item for any $w_0\in W_n$,
    \[
    W_n=w_0+\Ker(A_n),
    \]
    so $\dim W_n=d-\rank(A_n)$;
    \item for a query $q\in\F^d$, all $w\in W_n$ induce the same answer $w^{\top}q$ if and only if
    \[
    q\in\Row(A_n);
    \]
    \item if $\rank(A_n)=d$, then $W_n$ is a singleton and every query is determined;
    \item the number of remaining latent parameters is
    \[
    |W_n|=|\F|^{d-\rank(A_n)}.
    \]
\end{enumerate}
\end{theorem}

\begin{proof}
Fix $w_0\in W_n$. For $w\in W_n$, $A_n(w-w_0)=0$, so $w-w_0\in\Ker(A_n)$. Conversely, if $u\in\Ker(A_n)$, then $A_n(w_0+u)=b^{(n)}$, so $w_0+u\in W_n$. This proves (i). The dimension statement follows from rank-nullity.

For (ii), assume first $q\in\Row(A_n)$. Then $q=A_n^{\top}\alpha$ for some $\alpha\in\F^n$. If $w_1,w_2\in W_n$, then $u=w_1-w_2\in\Ker(A_n)$, hence
\[
(w_1-w_2)^{\top}q=u^{\top}A_n^{\top}\alpha=(A_nu)^{\top}\alpha=0.
\]
Thus all consistent parameters agree on $q$.

Conversely, if $q\notin\Row(A_n)$, then $q\notin\Ker(A_n)^{\perp}$. Hence there exists $u\in\Ker(A_n)$ with $u^{\top}q\neq0$. The two parameters $w_0$ and $w_0+u$ both lie in $W_n$, but they give different values on $q$. This proves (ii).

Item (iii) follows because $\rank(A_n)=d$ gives $\Ker(A_n)=\{0\}$. Item (iv) follows because an affine subspace over $\F$ of dimension $d-\rank(A_n)$ has $|\F|^{d-\rank(A_n)}$ elements.
\end{proof}

This example distinguishes global identification from query-local identification. A context may fail to determine the whole latent parameter while still determining the queried answer. The condition $q\in\Row(A_n)$ is the exact finite certificate.

\begin{corollary}[Residual entropy in the uniform linear case]\label{cor:entropy}
Assume the posterior on $W_n$ is uniform. The residual entropy of the latent parameter is
\[
H(w\mid E_{\le n})=(d-\rank(A_n))\log |\F|.
\]
The residual entropy of the scalar answer $w^{\top}q$ is zero exactly when $q\in\Row(A_n)$.
\end{corollary}

\begin{proof}
The first formula follows from $|W_n|=|\F|^{d-\rank(A_n)}$. For the answer, Theorem~\ref{thm:linear} says that $w^{\top}q$ is constant on $W_n$ exactly when $q\in\Row(A_n)$. A constant random variable has entropy zero. If $q\notin\Row(A_n)$, the same theorem gives at least two possible answer values, so the entropy is positive under a full-support uniform posterior on $W_n$.
\end{proof}

\begin{theorem}[Counting unresolved hypotheses and queries]\label{thm:linearcount}
Let $|\F|=Q$ and assume $W_n\neq\varnothing$. If $r=\rank(A_n)$, then:
\begin{enumerate}[label=(\roman*),leftmargin=2.5em]
    \item the number of latent parameters consistent with the context is
    \[
    |W_n|=Q^{d-r};
    \]
    \item the set of queries whose answers are already determined is exactly $\Row(A_n)$ and has size $Q^r$;
    \item the number of query vectors whose answers remain underdetermined is
    \[
    Q^d-Q^r.
    \]
\end{enumerate}
\end{theorem}

\begin{proof}
By Theorem~\ref{thm:linear}, $W_n$ is an affine subspace parallel to $\Ker(A_n)$, and $\dim\Ker(A_n)=d-r$. Therefore $|W_n|=Q^{d-r}$. The same theorem states that a query is determined exactly when it lies in $\Row(A_n)$. Since the row space has dimension $r$, it contains $Q^r$ vectors. The remaining $Q^d-Q^r$ vectors are outside the row space and therefore underdetermined.
\end{proof}

\begin{theorem}[Random-context identification probability]\label{thm:randomlinear}
Let the example inputs $a_1,\ldots,a_n$ be independent uniform vectors in $\F^d$, with $|\F|=Q$, and let outputs be generated by a fixed latent parameter $w^\ast$ through $b_i=(w^\ast)^\top a_i$. Then full semantic identification occurs exactly when $A_n$ has rank $d$, and
\[
\Prb(\rank(A_n)=d)=
\begin{cases}
0,& n<d,\\[2mm]
\displaystyle\prod_{i=0}^{d-1}(1-Q^{i-n}),& n\ge d.
\end{cases}
\]
Consequently, the probability that every query is semantically determined after $n$ random examples is given by the same expression.
\end{theorem}

\begin{proof}
The consistency set is
\[
W_n=\{w\in\F^d:A_nw=A_nw^\ast\}=w^\ast+\Ker(A_n).
\]
Thus $W_n$ is a singleton exactly when $\Ker(A_n)=\{0\}$, equivalently $\rank(A_n)=d$. If $n<d$, full rank is impossible. If $n\ge d$, the number of full-column-rank $n\times d$ matrices over $\F$ is
\[
(Q^n-1)(Q^n-Q)\cdots(Q^n-Q^{d-1}).
\]
Dividing by the total number $Q^{nd}$ of $n\times d$ matrices gives
\[
\prod_{i=0}^{d-1}\frac{Q^n-Q^i}{Q^n}
=
\prod_{i=0}^{d-1}(1-Q^{i-n}).
\]
The last statement follows from Theorem~\ref{thm:linear}: full rank is exactly the condition under which all queries are determined.
\end{proof}

This gives a genuine certificate-complexity statement rather than a philosophical analogy. The context length $n$ controls the rank distribution of the observed design matrix, and the rank distribution exactly controls semantic identifiability.

For many evaluations, full recovery of $w^\ast$ is stronger than necessary. A fixed query can become determined before the latent parameter is globally identified. The next theorem gives the query-local analogue of the full identification curve.

\begin{definition}[Rank distribution]
For $0\le r\le \min(n,d)$, let
\[
R_{n,d,Q}(r):=\frac{1}{Q^{nd}}
\frac{\prod_{i=0}^{r-1}(Q^n-Q^i)(Q^d-Q^i)}{\prod_{i=0}^{r-1}(Q^r-Q^i)}
\]
with the empty product interpreted as $1$. This is the probability that a uniformly random $n\times d$ matrix over $\mathbb F_Q$ has rank $r$.
\end{definition}

\begin{theorem}[Query-local identification probability]\label{thm:querylocal}
Let $q\in\mathbb F_Q^d$ be nonzero and let $A_n$ have independent uniform rows in $\mathbb F_Q^d$. Then
\[
\Prb\bigl(q\in\Row(A_n)\bigr)
=
\sum_{r=0}^{\min(n,d)} R_{n,d,Q}(r)\,\frac{Q^r-1}{Q^d-1}.
\]
Equivalently, the probability that the answer to the fixed query $q$ is semantically determined after $n$ random examples is the expression above. The zero query is determined with probability $1$ for every $n$.
\end{theorem}

\begin{proof}
Condition on the event $\rank(A_n)=r$. By symmetry, conditional on this event the row space of $A_n$ is uniformly distributed among all $r$-dimensional subspaces of $\mathbb F_Q^d$. The number of nonzero vectors in such a subspace is $Q^r-1$, while the total number of nonzero vectors in $\mathbb F_Q^d$ is $Q^d-1$. Hence for fixed nonzero $q$,
\[
\Prb(q\in\Row(A_n)\mid \rank(A_n)=r)=\frac{Q^r-1}{Q^d-1}.
\]
Averaging over the rank distribution gives the stated formula. The equivalence with semantic determination follows from Theorem~\ref{thm:linear}(ii). The zero query belongs to every row space, including the zero subspace.
\end{proof}

\begin{corollary}[Expected fraction of determined queries]\label{cor:expecteddetermined}
For random $A_n$ over $\mathbb F_Q$, the expected fraction of all query vectors whose answers are determined is
\[
\E\left[Q^{\rank(A_n)-d}\right]
=
\sum_{r=0}^{\min(n,d)} R_{n,d,Q}(r)Q^{r-d}.
\]
\end{corollary}

\begin{proof}
By Theorem~\ref{thm:linearcount}, if $\rank(A_n)=r$ then exactly $Q^r$ of the $Q^d$ query vectors are determined. Divide by $Q^d$ and average over $r$.
\end{proof}

\subsection{Definable hypothesis quotients}

The linear case is only a witness. The model-theoretic pattern is a quotient of hypotheses by the answers they induce.

\begin{definition}[Hypothesis-answer equivalence]
Let $\mathcal{H}$ be a definable family of structures or latent parameters, and let $Q$ be a set of queries. For $h,h'\in\mathcal{H}$ define
\[
h\equiv_Q h'
\quad\Longleftrightarrow\quad
(\forall q\in Q)\ f_h(q)=f_{h'}(q).
\]
For a finite example set $E_{\le n}$, define the version space
\[
\mathcal{H}_n:=\{h\in\mathcal{H}:h\text{ is consistent with }E_{\le n}\}.
\]
\end{definition}

\begin{proposition}[Query determination by quotient collapse]\label{prop:quotientcollapse}
For a single query $q$, the context $E_{\le n}$ determines the answer to $q$ if and only if $\mathcal{H}_n$ is contained in one $\equiv_{\{q\}}$-equivalence class.
\end{proposition}

\begin{proof}
If $\mathcal{H}_n$ is contained in one equivalence class, all hypotheses in $\mathcal{H}_n$ give the same value on $q$, so the answer is determined. Conversely, if the answer is determined, then any two hypotheses in $\mathcal{H}_n$ agree on $q$, hence belong to the same $\equiv_{\{q\}}$-class.
\end{proof}

\begin{theorem}[Finite quotient certificate]\label{thm:quotientfinite}
Suppose $\mathcal{H}_\omega:=\bigcap_n\mathcal{H}_n$ is contained in one $\equiv_{\{q\}}$-class and the consistency of $h\in\mathcal{H}_n$ is first-order expressible over the expansion chain. Then there exists $N$ such that $\mathcal{H}_N$ is contained in one $\equiv_{\{q\}}$-class.
\end{theorem}

\begin{proof}
The assumption says that the direct-limit theory entails equality of the answer on $q$ across all remaining hypotheses. Written as a first-order sentence in the expanded language, this is an entailment of $T_\omega$. By compactness, a finite subset of $T_\omega$ already entails it. Since the theory chain is increasing, that finite subset is contained in some $T_N$. Therefore $\mathcal{H}_N$ has already collapsed to one answer class for $q$.
\end{proof}

This theorem expresses the same idea as the linear row-space criterion, but without assuming linearity. Examples eliminate answer-equivalence classes. In-context learning is exact when the remaining classes collapse for the query under consideration.

\section{Emergence as asymptotic latent entailment}\label{sec:emergence}

Scale introduces a directed family of semantic presentations. We suppress the context parameter in this section when it is fixed. For each sentence $\varphi$, write
\[
s_{\lambda}(\varphi):=s_{\lambda,c_{\lambda}}(\varphi),
\]
where $c_\lambda$ may itself depend on scale.

\begin{definition}[Decisive scale family and limit theory]
Let $\mathcal{F}\subseteq\Sent(\Lang)$ be closed under negation and finite conjunction. A scale family is \emph{decisive on $\mathcal{F}$} if for every $\varphi\in\mathcal{F}$ the limit
\[
\lim_{\lambda}s_{\lambda}(\varphi)
\]
exists and belongs to $\{0,1\}$. The almost-sure limit theory on $\mathcal{F}$ is
\[
T_{\infty}^{\mathcal F}:=\{\varphi\in\mathcal{F}:\lim_{\lambda}s_{\lambda}(\varphi)=1\}.
\]
When $\mathcal{F}=\Sent(\Lang)$, write $T_{\infty}$.
\end{definition}

\begin{theorem}[Limit theory is complete and consistent]\label{thm:limit}
If the family is decisive on $\mathcal{F}$, then $T_{\infty}^{\mathcal F}$ is complete on $\mathcal{F}$ and every finite subset of it is satisfiable at all sufficiently large scales.
\end{theorem}

\begin{proof}
For completeness, fix $\varphi\in\mathcal{F}$. Since
\[
s_{\lambda}(\neg\varphi)=1-s_{\lambda}(\varphi),
\]
exactly one of $\varphi$ and $\neg\varphi$ has limiting confidence $1$.

For finite satisfiability, let $\Delta=\{\delta_1,\ldots,\delta_m\}\subseteq T_{\infty}^{\mathcal F}$ and put $\delta=\bigwedge_i\delta_i$. Since $\mathcal{F}$ is closed under finite conjunction, $\delta\in\mathcal{F}$. By subadditivity,
\[
1-s_{\lambda}(\delta)\le \sum_{i=1}^m(1-s_{\lambda}(\delta_i)).
\]
Each summand tends to $0$, so $s_{\lambda}(\delta)\to1$. Hence $s_{\lambda}(\delta)>0$ at all sufficiently large scales. Therefore some structure at those scales satisfies $\Delta$.
\end{proof}

\begin{definition}[Threshold theory and manifestation]
For $\tau\in(0,1)$ define
\[
T_{\lambda,\tau}:=\{\varphi\in\Sent(\Lang):s_{\lambda}(\varphi)>\tau\}.
\]
A sentence $\varphi$ is a \emph{latent entailment} when $\varphi\in T_{\infty}$. It is \emph{$\tau$-manifest from scale $\lambda_0$} when
\[
\varphi\in T_{\lambda,\tau}
\qquad\text{for all }\lambda\ge\lambda_0.
\]
It is \emph{$\tau$-emergent} when it is latent, hidden below threshold at some smaller scale, and $\tau$-manifest from some later scale onward.
\end{definition}

\begin{theorem}[Threshold manifestation of latent entailments]\label{thm:manifest}
If $\varphi\in T_{\infty}$ and $\tau<1$, then $\varphi$ is $\tau$-manifest from some scale onward. If it is below threshold at some earlier scale, then it is $\tau$-emergent.
\end{theorem}

\begin{proof}
Since $\varphi\in T_{\infty}$, $s_{\lambda}(\varphi)\to1$. For $\tau<1$, convergence gives a scale $\lambda_0$ such that $s_{\lambda}(\varphi)>\tau$ for all $\lambda\ge\lambda_0$. The second statement is the definition of $\tau$-emergence.
\end{proof}

This definition makes emergence a relation between a limit theory and an observation threshold. It is not an unexplained creation of a semantic fact at large scale.

\begin{theorem}[Rate-sensitive threshold bound]\label{thm:ratebound}
Assume $\varphi\in T_\infty$ and suppose there are constants $a>0$, $\alpha>0$, and $\lambda_1$ such that
\[
1-s_\lambda(\varphi)\le a\lambda^{-\alpha}
\qquad\text{for all }\lambda\ge\lambda_1.
\]
Then for every threshold $\tau<1$,
\[
\varphi\in T_{\lambda,\tau}
\]
for all
\[
\lambda\ge \max\left\{\lambda_1,\left(\frac{a}{1-\tau}\right)^{1/\alpha}\right\}
\]
after an arbitrarily small enlargement of the right-hand side if the threshold is strict.
\end{theorem}

\begin{proof}
If $\lambda$ satisfies the displayed lower bound before the final strict-threshold adjustment, then
\[
a\lambda^{-\alpha}\le 1-\tau.
\]
Therefore
\[
1-s_\lambda(\varphi)\le 1-\tau,
\]
so $s_\lambda(\varphi)\ge\tau$. For a strict threshold, replace the displayed lower bound by $(a/(1-\tau-\epsilon))^{1/\alpha}$ for any $\epsilon\in(0,1-\tau)$. Then $s_\lambda(\varphi)>\tau$.
\end{proof}

This theorem converts a qualitative limit-theory statement into a crossing-scale estimate. It is the form needed for theoretical comparison with empirical scaling curves: a rate of semantic convergence implies a bound on the scale at which the property becomes visible to a chosen threshold.

\subsection{Ultraproduct representation of scale limits}

The asymptotic commitments can be represented by a single structure through an ultraproduct. The following construction avoids assuming that one fixed structure satisfies all limit sentences at each finite scale.

\begin{theorem}[Ultraproduct witness for the limit theory]\label{thm:ultra}
Let $T_{\infty}$ be countable. Suppose for every finite $\Delta\subseteq T_{\infty}$ there exists $N_\Delta$ such that for all $n\ge N_\Delta$ some $\M_n\in\mathcal{K}_n$ satisfies $\Delta$. Then there are structures $\M_n\in\mathcal{K}_n$ and a nonprincipal ultrafilter $\mathcal{U}$ on $\N$ such that
\[
\prod_{n\to\mathcal U}\M_n\models T_{\infty}.
\]
\end{theorem}

\begin{proof}
Enumerate $T_{\infty}=\{\varphi_1,\varphi_2,\ldots\}$ and let
\[
\Delta_k:=\{\varphi_1,\ldots,\varphi_k\}.
\]
For each $n$, choose $k(n)$ maximal such that $N_{\Delta_{k(n)}}\le n$, taking $k(n)=0$ if no such $k$ exists. Choose $\M_n\in\mathcal{K}_n$ satisfying $\Delta_{k(n)}$ whenever $k(n)>0$, and choose any element of $\mathcal{K}_n$ otherwise. For each fixed $j$, all sufficiently large $n$ satisfy $k(n)\ge j$, so
\[
\{n:\M_n\models\varphi_j\}
\]
contains a cofinite set. Every nonprincipal ultrafilter contains all cofinite sets. By \L{}o\'s's theorem, the ultraproduct satisfies each $\varphi_j$.
\end{proof}

The ultraproduct is a limit witness. It packages the eventual finite satisfiability of semantic commitments into one mathematical object.

\subsection{Metric thresholds and apparent jumps}

A benchmark score is an observation of semantic confidence through a metric. If the metric is discontinuous, it can create threshold jumps without any discontinuity in the score being measured.

\begin{definition}[Observation functional]\label{def:obsfunctional}
An observation functional is a map
\[
\Omega:[0,1]\to\R.
\]
It is \emph{thresholded} at $\tau$ if
\[
\Omega_{\tau}(x)=\begin{cases}
0,&x<\tau,\\
1,&x\ge\tau.
\end{cases}
\]
\end{definition}

\begin{theorem}[Threshold jumps do not imply semantic discontinuity]\label{thm:mirage}
Let $s:\N\to[0,1]$ be nondecreasing and $s(n)\to1$. Fix $\tau\in(0,1)$ and define $m_{\tau}(n)=\Omega_{\tau}(s(n))$. If there exists $N$ such that
\[
s(N-1)<\tau\le s(N),
\]
then $m_{\tau}$ jumps at $N$. This can occur even when every increment $s(n)-s(n-1)$ is at most any prescribed $\epsilon>0$.
\end{theorem}

\begin{proof}
The displayed inequality gives $m_{\tau}(N-1)=0$ and $m_{\tau}(N)=1$. For the final claim, fix $\epsilon>0$ and choose an increasing sequence from $0$ to $1$ whose mesh is at most $\epsilon$, for instance a sufficiently fine discretization of $1-e^{-t}$. It crosses $\tau$ at some index, so the thresholded metric jumps, while the underlying increments are bounded by $\epsilon$.
\end{proof}

\begin{proposition}[Continuous observation preserves graduality]\label{prop:continuousobs}
Let $\Omega:[0,1]\to\R$ be uniformly continuous. For every $\eta>0$ there exists $\epsilon>0$ such that if
\[
|s(n)-s(n-1)|<\epsilon,
\]
then
\[
|\Omega(s(n))-\Omega(s(n-1))|<\eta.
\]
\end{proposition}

\begin{proof}
This is precisely uniform continuity applied to the two points $s(n)$ and $s(n-1)$.
\end{proof}

The mathematical distinction is now explicit. A jump in a thresholded metric is evidence about the metric and the crossing point. It is not, by itself, evidence that semantic confidence changed discontinuously.

A common way for a benchmark to manufacture a threshold is to score an answer as correct only when several fields are simultaneously correct. Exact match over a multi-field certificate is exactly such a conjunction, and it sharpens the apparent jump in a way the next proposition makes precise; it is the mechanism observed in Section~\ref{sec:benchmark}.

\begin{proposition}[Conjunctive thresholding of exact match]\label{prop:conj}
Let an item require $k$ answer fields, and suppose a generator returns each field correctly and independently with probability $s\in[0,1]$. Score the item by its graded value, the expected fraction of correct fields, and by exact match, the event that all $k$ fields are correct. Then:
\begin{enumerate}[label=(\roman*),leftmargin=2.5em]
\item the graded value is $s$ and the exact-match probability is $s^{k}$, so $s^{k}\le s$, with equality iff $k=1$ or $s\in\{0,1\}$;
\item the rise of exact match is compressed toward $s=1$: for every $\eta\in(0,1)$, $s^{k}<\eta$ whenever $s<\eta^{1/k}$, and the width $1-\eta^{1/k}$ of that interval tends to $0$ as $k\to\infty$;
\item the exact match of one item is the conjunction of its $k$ field-correctness events, hence a thresholded observation in the sense of Definition~\ref{def:obsfunctional}; by Theorem~\ref{thm:mirage} it can flip from $0$ to $1$ while the graded field-confidence rises smoothly, and aggregated over items the exact accuracy follows the compressed envelope $s^{k}$ of (i)--(ii). For a single field $(k=1)$ the observation is the identity and Proposition~\ref{prop:continuousobs} applies, so no compression occurs;
\item if moreover $1-s_{\lambda}\le a\lambda^{-\alpha}$ for a scale parameter $\lambda$ (the rate hypothesis of Theorem~\ref{thm:ratebound}), then exact match on a $k$-field item clears a threshold $\tau$ only once
\[
\lambda\ge\Bigl(\tfrac{ka}{1-\tau}\Bigr)^{1/\alpha},
\]
a factor $k^{1/\alpha}$ later than the single-field crossing scale of Theorem~\ref{thm:ratebound}.
\end{enumerate}
\end{proposition}

\begin{proof}
(i) By independence the all-correct probability is $\prod_{j=1}^{k}s=s^{k}$, while the expected fraction of correct fields is $\frac1k\sum_{j=1}^{k}s=s$; the inequality and its equality cases are elementary. (ii) $s^{k}<\eta\iff s<\eta^{1/k}$, and $\eta^{1/k}\to1$ as $k\to\infty$. (iii) The all-correct indicator is the meet of the $k$ field indicators, a thresholded observation; Theorem~\ref{thm:mirage} gives the flip from a smooth underlying confidence, and averaging the indicator over a population with per-field reliability $s$ returns $s^{k}$. The case $k=1$ is the identity observation, covered by Proposition~\ref{prop:continuousobs}. (iv) By Bernoulli's inequality $(1-a\lambda^{-\alpha})^{k}\ge1-ka\lambda^{-\alpha}$, so $s_{\lambda}^{k}\ge\tau$ as soon as $ka\lambda^{-\alpha}\le1-\tau$, that is $\lambda\ge(ka/(1-\tau))^{1/\alpha}$.
\end{proof}

This is a purely metric effect: the underlying per-field confidence may rise by a fixed small amount while the exact-match score on a $k$-field item rises arbitrarily more steeply, and the scale at which it becomes visible is inflated by $k^{1/\alpha}$. Section~\ref{sec:benchmark} exhibits both halves on trained models, with the multi-field certificate families jumping and the single-field family staying gradual.

\subsection{Stability of answer sets under posterior perturbation}

The set-theoretic account above can be paired with probabilistic robustness. If two posterior semantic measures are close on definable events, then their confidence assignments are close on all formulas in the tested fragment.

\begin{definition}[Fragment total variation]
For two probability measures $\mu$ and $\nu$ on $\K$, and a fragment $\mathcal{F}\subseteq\Sent(\Lang)$, define
\[
\TV_{\mathcal F}(\mu,\nu):=
\sup_{\varphi\in\mathcal F}\left|\mu([\![\varphi]\!]_{\K})-\nu([\![\varphi]\!]_{\K})\right|.
\]
\end{definition}

\begin{proposition}[Confidence robustness]\label{prop:tvrobust}
If $\TV_{\mathcal F}(\mu,\nu)\le\epsilon$, then for every $\varphi\in\mathcal F$,
\[
|s_{\mu}(\varphi)-s_{\nu}(\varphi)|\le\epsilon,
\]
where $s_{\mu}(\varphi)=\mu([\![\varphi]\!]_{\K})$ and similarly for $\nu$.
\end{proposition}

\begin{proof}
This is immediate from the definition of the supremum.
\end{proof}

\begin{corollary}[Threshold stability margin]\label{cor:thresholdmargin}
Let $\varphi\in\mathcal F$ and suppose $s_{\mu}(\varphi)>\tau+\epsilon$. If $\TV_{\mathcal F}(\mu,\nu)\le\epsilon$, then $s_{\nu}(\varphi)>\tau$.
\end{corollary}

\begin{proof}
By Proposition~\ref{prop:tvrobust}, $s_{\nu}(\varphi)\ge s_{\mu}(\varphi)-\epsilon>\tau$.
\end{proof}

This margin statement is the semantic analogue of robustness: once a property is sufficiently above threshold, small posterior perturbations cannot remove its threshold visibility.

\section{Verification, predictability, and finite certificates}\label{sec:verification}

The semantic calculus yields finite certificates in two directions: finite context certificates and finite scale thresholds. The first comes from compactness. The second comes from convergence in the limit theory.

\begin{theorem}[Finite property certificates]\label{thm:finitecert}
Let $q$ be a query term and let $\chi(y)$ be a formula with one free variable of sort $\mathbf{Y}$. If
\[
T_{\omega}\models\chi(f(q)),
\]
then there exists $N\in\N$ such that
\[
T_N\models\chi(f(q)).
\]
\end{theorem}

\begin{proof}
The entailment is equivalent to inconsistency of
\[
T_{\omega}\cup\{\neg\chi(f(q))\}.
\]
By compactness, a finite subset $\Sigma\subseteq T_{\omega}$ already gives the inconsistency. Since the chain is increasing, $\Sigma\subseteq T_N$ for some $N$. Therefore $T_N\models\chi(f(q))$.
\end{proof}

\begin{definition}[Prompt-stable property]
Let $\mathcal{P}$ be a class of prompt specifications on $X$. A sentence $\psi$ is \emph{stable over $\mathcal{P}$} if
\[
(\forall p\in\mathcal{P})(\forall\M\in\mathsf{Upd}_p(X))\ \M\models\psi.
\]
It is \emph{locally stable at $p$} if there exists a neighborhood $\mathcal{N}(p)\subseteq\mathcal{P}$ such that $\psi$ is stable over $\mathcal{N}(p)$.
\end{definition}

\begin{proposition}[Stability by invariant support]\label{prop:invariant}
If there is a set $Y\subseteq X$ such that
\[
\mathsf{Upd}_p(X)\subseteq Y
\qquad\text{for every }p\in\mathcal{P},
\]
and every $\M\in Y$ satisfies $\psi$, then $\psi$ is stable over $\mathcal{P}$.
\end{proposition}

\begin{proof}
For each $p\in\mathcal{P}$, all selected structures lie in $Y$, and every structure in $Y$ satisfies $\psi$. Hence all selected structures satisfy $\psi$.
\end{proof}

\begin{corollary}[Asymptotic predictability of certified properties]\label{cor:predictability}
Suppose a scale family is decisive and $\psi\in T_{\infty}$. Then for every $\tau<1$ there exists $\lambda_0$ such that
\[
\psi\in T_{\lambda,\tau}
\qquad\text{for all }\lambda\ge\lambda_0.
\]
If $\psi$ is also entailed by an in-context limit theory $T_{\omega}$, then there exists a finite context stage $N$ such that $T_N\models\psi$.
\end{corollary}

\begin{proof}
The first statement is Theorem~\ref{thm:manifest}. The second statement is Theorem~\ref{thm:finitecert} applied to the relevant formula.
\end{proof}

This separates two verification questions. A finite context stage can certify a property at the logical level. A sufficiently large scale can make that property visible above a confidence threshold. Neither statement requires treating the model as transparent at the level of weights.

\subsection{Certificate-checked refinement}\label{sec:refinement}

The certificates above are finite, and for the three families they are also \emph{checkable without the answer they certify}. This makes a closed-loop use of them possible, in which checking, rather than a reference key, drives improvement.

\begin{definition}[Sound oracle-free checker]\label{def:checker}
For a family of instances $x$ with an admissible-answer relation, a \emph{checker} is a map $V$ sending a pair $(x,c)$ of an instance and a candidate certificate to $\{\mathsf{accept},\mathsf{reject}\}$ that is computable from $(x,c)$ alone. It is \emph{sound} if $V(x,c)=\mathsf{accept}$ implies that the answer encoded by $c$ is correct, and \emph{oracle-free} in that it never reads a reference answer.
\end{definition}

\begin{lemma}[The family criteria are sound checkers]\label{lem:checkers}
Each family admits a sound oracle-free checker.
\begin{enumerate}[label=(\roman*),leftmargin=2.5em]
\item Linear family: accept a forced certificate $(c,v)$ when $\sum_i c_i a_i=q$ and $v=\sum_i c_i b_i$ over $\F$, and an underdetermination certificate $(w_1,w_2)$ when both satisfy $w^\top a_i=b_i$ for all $i$ and $w_1^\top q\neq w_2^\top q$.
\item Threshold family: accept $(\lambda,\delta)$ when $\lambda\ge2$ is an integer, $s_{\lambda}\ge\tau>s_{\lambda-1}$, and $\delta=s_{\lambda}-s_{\lambda-1}$.
\item Preferential family: accept $(S_p,S_{p\oplus q},b)$ when $S_p$ and $S_{p\oplus q}$ are the rank-minimal selected sets and the bit $b$ equals $[\,S_{p\oplus q}\subseteq S_p\,]$.
\end{enumerate}
\end{lemma}

\begin{proof}
Each acceptance condition restates the corresponding criterion. (i) is Theorem~\ref{thm:linear}(ii): a row-space combination with $c^\top A=q^\top$ forces $w^\top q=c^\top b$ for every consistent $w$, while two consistent witnesses separating $q$ exhibit underdetermination. (ii) is the crossing-scale condition behind Theorem~\ref{thm:ratebound}: $s_{\lambda}\ge\tau>s_{\lambda-1}$ identifies $\lambda$ as the least crossing scale and fixes the local increment. (iii) is the finite case of Theorem~\ref{thm:preservation}, whose preservation bit is exactly the containment of selected sets. Every condition is a function of $x$ and $c$, so $V$ is oracle-free.
\end{proof}

\begin{definition}[Aversive refinement schedule]\label{def:refine}
Fix a sound checker $V$ and a budget $K$. A refinement maintains an \emph{accepted set} $A_t$ of items, each carrying a verified certificate. The set $A_0$ is whatever a first pass certifies. At round $t\ge1$ a generator proposes certificates for the items outside $A_{t-1}$, conditioned on a memory $M_t$ of its own previously rejected certificates together with the checker's rejection reasons for them; $V$ is applied, and items it accepts enter $A_t$ and are never revised. With $p_t$ the number of still-unverified items, the schedule greedily increases
\[
J=\sum_{t}\bigl(|A_t|-\kappa\,p_t+\nu\,z_t\bigr),\qquad z_t=|A_{t-1}|,
\]
in which a rejection is the aversive term $p$ and the retained certificates are the continuity term $z$; accepted certificates are never revised, so there is no disruption term.
\end{definition}

\begin{proposition}[Monotone soundness of refinement]\label{prop:refine}
For every generator and budget $K$:
\begin{enumerate}[label=(\roman*),leftmargin=2.5em]
\item $A_0\subseteq A_1\subseteq\cdots\subseteq A_K$;
\item every answer in each $A_t$ is correct;
\item $|A_t|$ is non-decreasing and is a lower bound on the generator's true accuracy;
\item the schedule consumes no reference answer, so the gain $|A_K|-|A_0|$ counts additional correct certificates the generator can be driven to produce under checking, not information supplied by a key.
\end{enumerate}
\end{proposition}

\begin{proof}
(i) holds because accepted items are retained. (ii) is soundness of $V$ (Lemma~\ref{lem:checkers}). (iii) follows from (i) and (ii). For (iv), $V$ depends only on $(x,c)$ (Definition~\ref{def:checker}), and $M_t$ records only the generator's own rejected certificates and the checker's reasons, themselves functions of $(x,c)$; no external answer enters the loop.
\end{proof}

The count $p_t$ is a logical violation cost in the sense of Definition~\ref{def:violationcost}, and the schedule is its greedy minimization under the preferential reading of Section~\ref{sec:prompt}: each round retains the cost-zero, that is verified, proposals and carries them forward.

Soundness fixes that every accepted answer is correct, but not how fast $A_t$ grows; that rate is set by how the generator uses the checker's reason. The threshold family makes this exact, because its rejection reason is \emph{directional}: a proposed scale $\lambda$ is told that it is too small when $s_\lambda<\tau$, too large when $s_{\lambda-1}\ge\tau$, or, once $\lambda$ is the crossing scale, the value the increment must take.

\begin{proposition}[Reach of the refinement]\label{prop:reach}
Take a threshold item whose crossing scale $\lambda^\ast$ lies in $[2,R]$ with $R=\lceil(a/(1-\tau))^{1/\alpha}\rceil$ (Theorem~\ref{thm:ratebound}), under the directional checker of Lemma~\ref{lem:checkers}(ii) and a budget $K$.
\begin{enumerate}[label=(\roman*),leftmargin=2.5em]
\item A generator that bisects the interval still consistent with the directional replies certifies the item within $\lceil\log_2(R-1)\rceil+1$ rounds; for $K$ at least this large the schedule reaches the whole family.
\item A generator that merely never re-proposes a refuted scale certifies it within $R-1$ rounds.
\item A generator whose proposal in each round is conditionally independent of the reply, accepting with probability at most $q$, certifies the item by round $K$ with probability at most $1-(1-q)^K$; if it cannot localize $\lambda^\ast$ from the signal then $q$ is the chance of jointly guessing the scale and the increment, and the expected $|A_K|$ stays near $|A_0|$.
\end{enumerate}
\end{proposition}

\begin{proof}
The rate hypothesis $1-s_\lambda\le a\lambda^{-\alpha}$ gives $\lambda^\ast\le R$, so the search ranges over $[2,R]$. (i) Each directional reply discards the half of the feasible interval on the refuted side of the proposed midpoint, so after $\lceil\log_2(R-1)\rceil$ rounds the interval is $\{\lambda^\ast\}$; one further round sets the increment, which the checker has already named. (ii) Each refuted proposal removes at least one scale from the feasible set, which holds fewer than $R-1$ members besides $\lambda^\ast$. (iii) Independence gives probability at least $1-q$ of no acceptance in a round, hence at most $1-(1-q)^K$ over $K$ rounds.
\end{proof}

So with a sound directional checker the schedule erases an apparent threshold for any generator able to act on the feedback (cases (i)--(ii)) and fails only for one blind to it (case (iii)). The dividing line is a property of the generator's search competence under the checker, not of the checker or of the certificate; by Proposition~\ref{prop:refine} soundness holds throughout, so wherever the loop advances the gain is genuine. This is the sense in which a finite, checkable certificate is not only a witness but a handle: where the underlying capability is latent it is exposed by a cheap oracle-free check, and where it is absent no amount of checking supplies it.

\section{Numerical verification of the theorems}\label{sec:verify-numeric}

The quantitative theorems are reproduced here by controlled, exact-arithmetic checks of the algebraic and threshold statements, in a setting where the latent task, the field, and the scoring metric are known exactly. These are artifact checks, not experiments on a trained model. They use finite-field arithmetic and a fixed seed, and the ancillary material contains \texttt{anc/verify\_theorems.py} together with the emitted CSV files and the two figures of this section. A separate, complementary question---whether contemporary trained systems can themselves emit the finite certificates that the theorems isolate---is taken up in Section~\ref{sec:benchmark}, and controlled measurement protocols for trained models are stated in Section~\ref{sec:predictions}.

\IfFileExists{fig_identification.pdf}{%
\begin{figure}[H]
\centering
\includegraphics[width=0.66\linewidth]{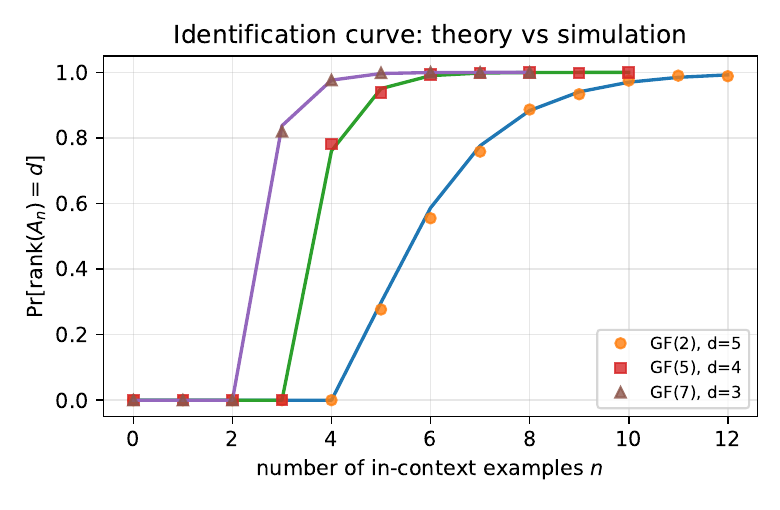}
\caption{Identification curve. Solid lines are the closed form $I_d(n)=\prod_{i=0}^{d-1}(1-Q^{i-n})$ of Theorem~\ref{thm:randomlinear}; markers are Monte-Carlo estimates of $\Pr[\rank(A_n)=d]$ over uniform contexts. The probability that every query becomes determined is the same curve.}
\label{fig:idcurve}
\end{figure}}{}

\paragraph{Identification curve.}
Table~\ref{tab:idcurve} compares the closed form $I_d(n)$ of Theorem~\ref{thm:randomlinear} with a Monte-Carlo estimate of $\Pr[\rank(A_n)=d]$ for uniform examples over $\F^d$. Across the configurations $(\F,d)\in\{(\mathbb F_2,5),(\mathbb F_5,4),(\mathbb F_7,3)\}$ and all context lengths $0\le n\le 2d+2$, the maximum absolute discrepancy is $0.032$ with $600$ trials per row, within the expected small-sample Monte-Carlo fluctuation. The curve is exactly zero for $n<d$ and rises sharply once $n\ge d$, matching the rank-nullity obstruction.

\begin{table}[H]
\centering
\begin{tabular}{rccc}
\toprule
$n$ & $I_d(n)$ (Theorem~\ref{thm:randomlinear}) & Monte Carlo & $|{\cdot}|$ \\
\midrule
$3$ & $0.0000$ & $0.0000$ & $0.000$\\
$4$ & $0.7606$ & $0.7817$ & $0.021$\\
$5$ & $0.9505$ & $0.9383$ & $0.012$\\
$6$ & $0.9900$ & $0.9933$ & $0.003$\\
$8$ & $0.9996$ & $1.0000$ & $0.000$\\
\bottomrule
\end{tabular}
\caption{Theory versus simulation for $\F=\mathbb F_5$, $d=4$ ($600$ trials per row). Full-determination probability matches the identification curve.}
\label{tab:idcurve}
\end{table}

\paragraph{Row-space and query-local determinacy.}
Theorem~\ref{thm:linear}(ii) states that a query $q$ is determined by a context exactly when $q\in\Row(A_n)$. Over $\mathbb F_2^4$ we enumerate, for each sampled context, all latent parameters consistent with the labels and test whether they agree on $w^\top q$. Across $2{,}400$ context--query pairs emitted by the artifact, the predicate ``all consistent parameters agree on $q$'' coincides with ``$q\in\Row(A_n)$'' in every case: zero mismatches. The file \texttt{anc/fixed\_query\_curve.csv} additionally records the fixed-query curve of Theorem~\ref{thm:querylocal}, and \texttt{anc/rowspace\_determinacy.csv} records the expected determined-query fraction of Corollary~\ref{cor:expecteddetermined}.

\paragraph{Threshold mirage and crossing scale.}
Take the smooth confidence $s_\lambda(\varphi)=1-a\lambda^{-\alpha}$ with $a=1$, $\alpha=1$, and threshold $\tau=0.9$. The thresholded metric $\Omega_\tau(s_\lambda)$ flips from $0$ to $1$ at $\lambda=10$, exactly the crossing-scale bound $\lambda_\tau=(a/(1-\tau))^{1/\alpha}=10$ of Theorem~\ref{thm:ratebound}, while the increment of $s_\lambda$ at the crossing is only $s_{10}-s_{9}\approx 0.011$. A discontinuous benchmark jump is therefore produced by a confidence that changes by about one percentage point, confirming Theorem~\ref{thm:mirage}; the same $s_\lambda$ read through a continuous metric stays gradual, confirming Proposition~\ref{prop:continuousobs}.

\IfFileExists{fig_mirage.pdf}{%
\begin{figure}[H]
\centering
\includegraphics[width=0.66\linewidth]{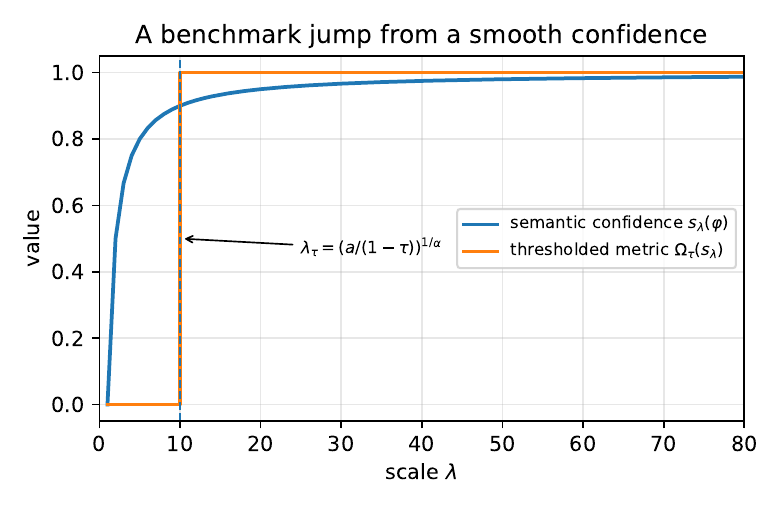}
\caption{A benchmark ``jump'' from a smoothly rising confidence. The thresholded metric (step) flips at the crossing scale $\lambda_\tau$ predicted by Theorem~\ref{thm:ratebound}, although $s_\lambda$ is smooth and changes by $\approx 0.011$ there.}
\label{fig:mirage}
\end{figure}}{}

\section{Certificate emission by trained language models}\label{sec:benchmark}

The simulations above check the theorems where the latent task and the metric are known in advance. A complementary question is whether the certificate objects the theorems isolate are within reach of trained systems: presented with an in-context problem, can a model return the same finite witness the theory uses---a row-space combination, a crossing-scale computation, or a selected-set preservation judgment? We probe this with a deterministic suite whose three item families mirror the three certificate types one to one. The suite is generated from a fixed seed, held identical across models, and graded against ground truth computed by exact arithmetic, so the questions and the answer key are fully reproducible offline.

\begin{itemize}[leftmargin=1.6em]
\item \textbf{DET} ($12$ items) instantiates the row-space criterion of Theorem~\ref{thm:linear}(ii) and the counts of Theorem~\ref{thm:linearcount}. Over $\mathbb F_p$ with $p\in\{2,5,7\}$ and dimensions $3\le d\le 5$, the model must decide whether a query is forced by an in-context linear example set and, when it is, return the forced value with a row-space combination as witness; underdetermined items require two consistent parameters that disagree on the query.
\item \textbf{THR} ($6$ items) instantiates the crossing-scale bound of Theorem~\ref{thm:ratebound} together with the anti-mirage statement of Theorem~\ref{thm:mirage}. For $s_\lambda=1-a\lambda^{-\alpha}$ and threshold $\tau$, the model must return the smallest integer scale at which $s_\lambda\ge\tau$ and the local increment $s_\lambda-s_{\lambda-1}$ there.
\item \textbf{PRS} ($5$ items) instantiates the preservation criterion of Theorem~\ref{thm:preservation}. In a finite preferential model with logically distinguishable worlds, the model must decide whether appending a prompt preserves every prior consequence, equivalently whether the newly selected set is contained in the old one.
\end{itemize}

To expose the transition rather than the ceiling, the panel is a weak-to-mid capability spread of nine systems from six independent laboratories. The frontier tier is omitted on purpose: it certifies the whole suite outright, so it sits past the transition and carries no information about where the transition is. Decoding is uniform---greedy, a high reasoning budget, and a single generous token ceiling so that answers are not truncated---and every response is graded only on its parsed final object, recovered by a brace-balanced extractor that is robust to surrounding reasoning text. Alongside the exact certificate score we record a \emph{graded} score, the mean fraction of an item's answer fields that are correct; this is the continuous proxy of Proposition~\ref{prop:conj}, whose threshold at $1$ returns the exact score. The credential is read only from the environment and is never written to disk. The runner, the exact prompt with its answer key, and the unedited responses are supplied as \texttt{anc/run\_certificate\_benchmark.py}, \texttt{anc/certificate\_benchmark\_protocol.txt}, and \texttt{anc/certificate\_benchmark\_responses.jsonl}.

\begin{table}[H]
\centering
\begin{tabular}{llcccrr}
\toprule
Model & Lab & DET & THR & PRS & Exact & Graded \\
& & ($/12$) & ($/6$) & ($/5$) & ($/23$) & ($/23$) \\
\midrule
\texttt{Qwen3 14B} & Alibaba & 12/12 & \textbf{6/6} & 4/5 & 22 & 22.0 \\
\texttt{Mistral Medium 3.1} & Mistral & 7/12 & 0/6 & 4/5 & 11 & 13.0 \\
\texttt{Mistral Small 3.2} & Mistral & 6/12 & 0/6 & 4/5 & 10 & 12.5 \\
\texttt{Llama 4 Scout} & Meta & 5/12 & 0/6 & 4/5 & 9 & 12.5 \\
\texttt{Qwen2.5 7B} & Alibaba & 6/12 & 0/6 & 4/5 & 10 & 11.5 \\
\texttt{Gemma 3 27B} & Google & 3/12 & 0/6 & 5/5 & 8 & 10.5 \\
\texttt{Llama 4 Maverick} & Meta & 4/12 & 0/6 & 3/5 & 7 & 9.0 \\
\texttt{Llama 3.1 8B} & Meta & 1/12 & 0/6 & 4/5 & 5 & 8.0 \\
\texttt{Claude 3.5 Haiku} & Anthropic & 3/12 & 0/6 & 3/5 & 6 & 7.5 \\
\bottomrule
\end{tabular}
\caption{Certificate emission across a weak-to-mid panel, sorted by graded score. \emph{Exact} counts items whose certificate is fully correct; \emph{Graded} is the mean fraction of correct answer fields (Proposition~\ref{prop:conj}). The threshold family \textbf{THR} is zero for every system except the strongest, which returns all six---an apparent emergent jump examined below.}
\label{tab:benchmark}
\end{table}

\IfFileExists{fig_benchmark.pdf}{%
\begin{figure}[H]
\centering
\includegraphics[width=0.74\linewidth]{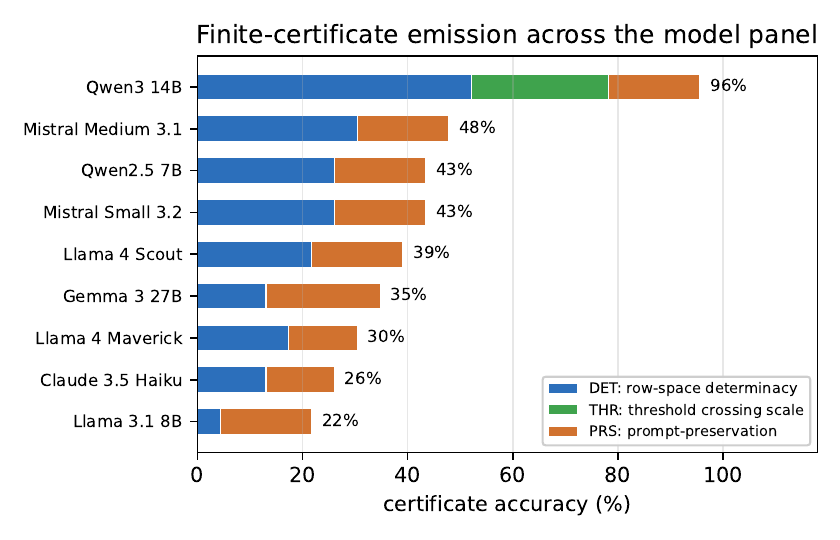}
\caption{Per-family exact certificate accuracy across the panel, as a fraction of the $23$-item suite. The threshold-crossing segment \textbf{THR} (green) is absent for every system but the strongest, where it appears in full.}
\label{fig:benchmark}
\end{figure}}{}

Two regularities appear. First, the graded score exceeds the exact score for every system, and the gap is widest on the multi-field families. Second, and more sharply, the threshold family behaves as an emergent capability: eight of the nine systems return \emph{no} correct threshold-crossing certificate, and the ninth returns all six. The exact \textbf{THR} score is thus flat at zero across most of the panel and then jumps to its maximum, while the graded \textbf{THR} confidence rises gradually through the panel ($0,\,0.5,\,1.0$ before the crossing). This is exactly the conjunctive threshold of Proposition~\ref{prop:conj}: a threshold-crossing certificate is accepted only when both the scale $\lambda$ and the local increment are correct ($k=2$), so its exact score tracks $s^{2}$ and remains near zero until the underlying field-confidence is high.

\IfFileExists{fig_benchmark_mirage.pdf}{%
\begin{figure}[H]
\centering
\includegraphics[width=0.66\linewidth]{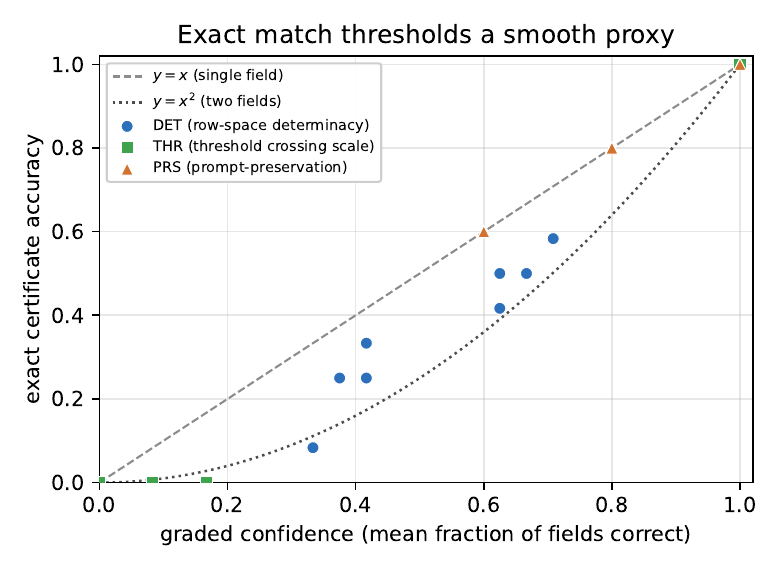}
\caption{Exact certificate accuracy against the graded proxy, one point per system per family. The single-field family \textbf{PRS} ($k=1$) lies on $y=x$; the two-field family \textbf{DET} follows $y=x^{2}$; \textbf{THR} sits at the foot of the $x^{2}$ curve until the single crossing. The apparent emergence is the metric artifact of Theorem~\ref{thm:mirage} and Proposition~\ref{prop:conj}, not a discontinuity in the underlying confidence.}
\label{fig:benchmarkmirage}
\end{figure}}{}

Figure~\ref{fig:benchmarkmirage} reads the same data through the two scores. The single-field family \textbf{PRS} lies on the diagonal $y=x$, exactly as Proposition~\ref{prop:conj}(iii) requires when no conjunction is present; the row-space family \textbf{DET}, dominated by two-field forced items, lies along $y=x^{2}$; and \textbf{THR} occupies the foot of the $x^{2}$ curve until the crossing. The emergence of threshold reasoning is therefore a thresholded-observation artifact in the precise sense of Theorem~\ref{thm:mirage}, reproduced on trained systems and quantified by the conjunction count $k$, rather than a discontinuity in semantic confidence. The witnesses returned are also of the intended form: for the underdetermined instance \texttt{DET6} over $\mathbb F_5$ a high-scoring system reported ``$q=(1,0,0)$ outside span; $w=(0,3,0)\!\mapsto\!0$ vs $w=(1,3,4)\!\mapsto\!1$,'' the two-witness certificate of Theorem~\ref{thm:linear}(ii), and for \texttt{PRS4} it reported ``$p$-selected $=\{W_1\}$, $(p\oplus q)$-selected $=\{W_2,W_3\}$; not a subset,'' the containment test of Theorem~\ref{thm:preservation}.

This probe is deliberately small and is not evidence for any theorem; the theorems are established by proof and reproduced by the exact-arithmetic checks above. Its role is to show that the certificate objects are concrete enough to be requested, produced, and machine-checked, and that the benchmark jump the theory predicts is visible on contemporary systems exactly where the conjunction count makes it sharpest.

\section{Predictions and a falsification protocol}\label{sec:predictions}

The theory gives three controlled measurement protocols for trained models. Each protocol specifies what must be measured and what would falsify the proposed semantic representation. Failure of a trained model to match a protocol would not refute the algebraic theorem; it would refute the claim that the tested behavior is well represented by that semantic task family or by that observation map.

\paragraph{Prediction 1 (row-space determinacy of in-context answers).}
For a model evaluated on a controlled linear in-context task family, answer concentration on a held-out query $q$ should track the row-space predicate $q\in\Row(A_n)$ if the model's behavior is represented by the finite-field semantic family. Full-context recovery should follow $I_d(n)$, while fixed-query recovery should follow Theorem~\ref{thm:querylocal}. \emph{Protocol.} Use synthetic linear probes over a controlled input dimension $d$; vary $n$; for each query record whether it is row-space spanned, whether the model's answer distribution is concentrated, and whether concentration follows the full or query-local curve. \emph{Falsifier.} Systematic concentration outside the row span, or systematic failure on spanned queries, falsifies this semantic representation for the tested model and task family.

\paragraph{Prediction 2 (emergence is a metric threshold, not a semantic discontinuity).}
A capability that appears to emerge under a thresholded or exact-match score should become gradual under a continuous semantic proxy, such as per-token log-probability of the gold answer, answer-distribution Brier score, calibrated confidence, or edit distance. The scale at which the thresholded score flips should be predictable from a fitted rate bound $1-s_\lambda\le a\lambda^{-\alpha}$. \emph{Protocol.} Re-score an emergence benchmark with both the original thresholded metric and a continuous proxy across model scales; fit the pre-threshold confidence curve; predict the crossing scale and compare it to the observed jump. \emph{Falsifier.} A discontinuity in the continuous proxy itself, after removing thresholding and discretization, would be evidence for a genuine semantic discontinuity rather than the metric artifact covered by Theorem~\ref{thm:mirage}.

\paragraph{Prediction 3 (prompt extension can delete consequences).}
Appending text to a prompt can remove previously supported conclusions when the appended text changes the selected preferred models. Which conclusions survive is given by the preservation criterion of Theorem~\ref{thm:preservation}: the surviving formulas are exactly those true throughout the newly selected set. \emph{Protocol.} Construct prompt pairs $(p,\,p\cdot q)$, estimate a fixed battery of high-confidence properties under $p$, and measure retention after appending $q$. \emph{Falsifier.} Uniform monotone retention under all prompt extensions would falsify the preferential-update representation for that prompt family; selective retention inconsistent with the selected-set criterion would falsify the proposed ranking model.

\section{Theoretical problems left open}\label{sec:directions}

The preceding sections settle qualitative determinacy, exact finite-field linear identification, prompt-preservation criteria, and threshold separation. They also expose several problems that are now sharply stated enough to be attacked directly.

\begin{problem}[Bounded-fragment certificate bounds]
Fix a fragment $\mathcal F$, a class of expansion chains, and a target schema $\chi(f(q))$. Determine the least function $B_{\mathcal F}(m,r)$ such that every entailment using $m$ examples and quantifier rank at most $r$ has a certificate of length at most $B_{\mathcal F}(m,r)$.
\end{problem}

The compactness theorem gives existence but no numerical bound. The finite-field results show that such bounds can be exact in algebraic task families. The next theoretical step is to connect certificate length with quantifier rank, VC-style dimension, Littlestone-style dimension, or rank invariants of the hypothesis quotient.

\begin{problem}[Preferential representation dimension]
Given a finite family of observed prompt consequence relations, characterize the smallest rank structure whose preferential models realize all of them.
\end{problem}

This is the prompt analogue of asking for a minimal automaton or minimal Kripke frame. A solution would quantify how much hidden priority structure is necessary to explain instruction-following behavior.

\begin{problem}[Fragment decisiveness under scale]
For a scale family $(\mathbb G_\lambda)_\lambda$ and a logical fragment $\mathcal F$, characterize when every $\varphi\in\mathcal F$ has a $0$-$1$ confidence limit.
\end{problem}

The limit theory in Section~\ref{sec:emergence} assumes decisiveness. The open mathematical issue is to derive decisiveness from structural hypotheses on the measured model classes, such as concentration, exchangeability, definable stability, or convergence of finite-dimensional marginals.

\begin{problem}[Decoder-faithfulness rates]
Given that $\psi\in T_\infty$, bound the rate at which the decoder failure probability for $\psi$ tends to zero, or construct examples where semantic convergence holds but decoder failure remains bounded away from zero.
\end{problem}

This problem matters because verification of a generated output requires both semantic entailment and decoder reliability. The present paper separates these notions; a full predictive theory must relate their rates under explicit assumptions.

\begin{problem}[Ultraproduct invariants of capability]
Identify which properties of the ultraproduct witness in Theorem~\ref{thm:ultra} correspond to stable capabilities across scale families.
\end{problem}

Candidate invariants include saturation level, omitted types, definable cuts, and elementary embeddings between scale subsequences. These are not benchmark labels; they are structural properties that may explain why some capabilities stabilize while others remain prompt-sensitive.

\section{Scope and failure modes}\label{sec:scope}

The construction is intentionally semantic rather than mechanistic. It does not claim to recover transformer circuits, training data, or internal activations. It also does not claim that the finite-field linear family is a universal model of in-context learning. Its claim is conditional and external: if a behavior is represented by a measured model class, a context update, and a decoder, then prompt consequence, finite context certification, and threshold manifestation obey the theorems proved above. This conditional form is necessary for falsifiability. A mismatch with trained-model measurements should be read as a failure of the proposed semantic representation for that behavior, not as a failure of first-order compactness or finite-field linear algebra.

This scope matters because several tempting interpretations are false. High benchmark accuracy does not imply membership in an almost-sure contextual theory. A prompt extension need not behave like monotone axiom addition. A finite example sequence can narrow admissible answers without identifying a unique latent task. A threshold jump can be produced by the scorer even when semantic confidence changes smoothly. Finally, a measure-one semantic property may still fail in sampled output if decoder faithfulness is weak.

The value of the calculus is therefore diagnostic. It states which extra assumption is needed for each inference. To pass from semantic entailment to observed reliability one needs a decoder-faithfulness bound. To pass from finite examples to unique prediction one needs a certificate such as the row-space condition in Theorem~\ref{thm:linear}. To pass from benchmark emergence to semantic emergence one needs evidence about the confidence function, not only about the thresholded score. These distinctions are exactly where verification claims about language models usually become ambiguous.

\section{Conclusion}

The paper has recast three empirical questions as finite semantic certification problems. A pre-trained language model is represented externally as a probability-bearing family of first-order structures equipped with context update and decoding. A prompt is a preferential update on a model class. In-context learning is a chain of semantic expansions by finite diagrams. Emergence is thresholded manifestation of membership in an almost-sure limit theory. The finite-field linear family supplies the fully explicit case: row-space membership is the certificate, rank is the obstruction, and the identification curves are closed-form.

The main mathematical mechanisms are elementary but targeted. Measure-one theories record semantic commitments. Definitional invariance prevents the construction from depending on notational choices. Preferential minima explain nonmonotonic prompt behavior and give exact preservation criteria. Compactness converts direct-limit entailment into finite context certificates, while indistinguishable-prefix arguments give lower bounds. Linear algebra supplies exact finite criteria, counting formulas, and random-context identification probabilities in a representative task family. Ultraproducts package asymptotic scale commitments into a single structure. Threshold theorems separate benchmark jumps from discontinuities in semantic confidence and give rate-sensitive crossing bounds.

This yields a clear verification picture. To verify a behavior, one should identify the formula expressing it, the context stage that entails it, the certificate size required to force it, the scale at which its semantic confidence clears the chosen threshold, and the decoder-faithfulness needed to make the semantic property visible in output. The ancillary package makes the quantitative part reproducible: it regenerates the identification curves, the row-space determinacy checks, and the threshold-mirage figure from the supplied Python script. Each component is explicit, and each can fail independently.

\end{document}